\documentclass{article} 
\usepackage{iclr2024_conference,times}

\usepackage{amsmath,amsfonts,bm}

\def\eqref#1{equation~\ref{#1}}

\def\1{\bm{1}}

\DeclareMathAlphabet{\mathsfit}{\encodingdefault}{\sfdefault}{m}{sl}
\SetMathAlphabet{\mathsfit}{bold}{\encodingdefault}{\sfdefault}{bx}{n}

\DeclareMathOperator*{\argmin}{arg\,min}

\usepackage{graphicx}
\usepackage{amsthm}
\usepackage{hyperref}
\usepackage{url}
\usepackage{comment}
\usepackage{bbm}
\usepackage{enumerate}
\usepackage{algorithm}
\usepackage{algorithmic}
\usepackage{dsfont}
\usepackage{amssymb}
\newtheorem{theorem}{Theorem}[section]

\newtheorem{lemma}[theorem]{Lemma}

\DeclareMathOperator{\e}{e}

\title{
On the Fairness ROAD: Robust Optimization for Adversarial Debiasing} 

\author{Vincent Grari\thanks{Equal contribution} $^{\,,1, 2, 4}$, Thibault Laugel\footnotemark[1] $^{\,,1, 2, 4}$,
Tatsunori Hashimoto$^2$,
Sylvain Lamprier$^3$,
Marcin Detyniecki$^{1, 4, 5}$ 
\\
$^1$ \quad AXA Group Operations\\ 
$^2$ \quad Stanford University\\
$^3$ \quad LERIA, Universit\'e d’Angers, France\\
$^4$ \quad TRAIL, Sorbonne Universit\'e, Paris, France\\
$^5$ \quad Polish Academy of Science, IBS PAN, Warsaw, Poland\\
\texttt{\{grari,laugel\}@stanford.edu} \\
\texttt{code: \url{https://github.com/fairmlresearch/ROAD}} \\
}

\newlength\myindent
\iclrfinalcopy 
\begin{document}

\maketitle

\begin{abstract}
In the field of algorithmic fairness, 
significant attention has been put on group fairness criteria, such as Demographic Parity and Equalized Odds. 
Nevertheless, these 
objectives, measured as global averages, 
have raised concerns about persistent local disparities between sensitive groups. 
In this work, we address the problem of local fairness, which ensures that the predictor is unbiased not only in terms of expectations over the whole population, but also within any subregion of the feature space, unknown at training time. 
To enforce this objective, we introduce ROAD, a novel approach that 
leverages
the Distributionally Robust Optimization (DRO) framework 
within a fair adversarial learning objective, where an adversary tries to predict the sensitive attribute from the predictions.
Using an instance-level re-weighting strategy, ROAD is designed to prioritize inputs that are likely to be locally unfair, i.e. where the adversary faces the least difficulty in reconstructing the sensitive attribute. 
Numerical experiments demonstrate the effectiveness of our method: it 
achieves,  for a given global fairness level, Pareto dominance with respect to 
local fairness and accuracy 
across three standard datasets, as well as enhances fairness generalization under distribution shift. 

\end{abstract}

\section{Introduction}

The increasing adoption of machine learning models in various applications, such as healthcare, finance, and criminal justice, has raised concerns about the fairness of algorithmic decision-making processes. As these models are often trained on historical data, they have been shown to unintentionally perpetuate existing biases and discrimination against certain vulnerable groups. Addressing fairness in machine learning has thus become an essential aspect of developing ethical and equitable systems, with the overarching goal of ensuring that prediction models are not influenced by sensitive attributes.
One of its most common concepts, group fairness, entails dividing the population into demographic-sensitive groups (e.g., male and female) and ensuring that the outcomes of a decision model are equitable across these different groups, as measured with criteria such as Demographic Parity (DP) ~\citep{dwork2012} and Equal Opportunity (EO) ~\citep{hardt2016}.


However, focusing solely on these group fairness criteria, along with predictive performance, has been increasingly questioned as an objective: besides being shown to poorly generalize to unseen, e.g. drifted, 
environments~\citep{kamp2021robustness}, it has been more generally criticized for being too simplistic~\citep{selbst2019fairness,binns2020apparent}, leading to arbitrariness in the bias mitigation process~\citep{krco2023mitigating} and the risk of having some people pay for others
~\citep{mittelstadt2023unfairness}. 
Recognizing these issues, some researchers have long focused on exploring more localized fairness behaviors, proposing to measure bias sectionally within predefined demographic categories, in which comparison between sensitive groups is deemed meaningful for the considered task. 
%
For instance, using \emph{Conditional Demographic Disparity}~\citep{vzliobaite2011handling}, fairness in predicted salaries between men and women 
shall be evaluated by comparing individuals within the same job category and seniority level, rather than making a global comparison across sensitive groups.

 Nevertheless, predefining these \emph{comparable} groups to optimize their local fairness is often difficult: for instance, which jobs should be deemed legally comparable with one another?~\citep{wachter2021why} 
In this paper, we therefore propose to address the difficult problem of enforcing fairness in local subgroups that are unknown at training time (Sec.~\ref{sec:background}). For this purpose, we leverage the Distributionally Robust Optimization (DRO) framework, initially proposed to address worst-case subgroup accuracy (see e.g.~\cite{duchi2018learning}). 
Our approach ROAD (\emph{Robust Optimization for Adversarial Debiasing}, described in Sec.\ref{sec:proposition-road}) combines DRO with a fair adversarial learning framework, which aims to minimize the ability of an adversarial model to reconstruct the sensitive attribute.
By boosting attention on feature regions where predictions are the most unfair in the sense of this sensitive reconstruction, ROAD is able to 
find the best compromise between local fairness, accuracy and global fairness.
Such dynamic focus is done by relying on a weighting process that respects some locality smoothness in the input space, in order to mitigate bias in any implicit subgroup of the population without  supervision. Experiments, described in Section~\ref{sec:experiments}, show the efficacy of the approach on various datasets. 

\section{Problem Statement}
\label{sec:background}

Throughout this document, we address a conventional supervised classification problem, trained using $n$ examples ${(x_{i},y_{i},s_{i})}_{i=1}^{n}$, 
where each example is composed of a feature vector $x_i \in \mathbb{R}^{d}$, containing $d$ predictors, a binary sensitive attribute $s_i$, and a binary label $y_i$. These examples are sampled from a training distribution $\Gamma=(X, Y, S) \sim p$. 
Our goal is to construct a predictive model~$f$ with parameters $w_f$ 
that minimizes the loss function $\mathcal{L_Y}(f(x), y)$ (e.g. log loss for binary classification), whilst adhering to fairness constraints based on specific fairness definitions relying on the sensitive attribute $S$. In this section, we present the fairness notions and works that are necessary to ground our proposition.



\subsection{Group Fairness}
\label{sec:background-groupfairness}

One key aspect of algorithmic fairness is group fairness, which aims to ensure that the outcomes of a decision model are equitable across different demographic groups. In this paper, we focus on two of the most well-known group fairness criteria: \emph{Demographic Parity} and \emph{Equalized Odds}.



\emph{Demographic Parity}: Demographic parity (DP)~\citep{dwork2012} is achieved when the proportion of positive outcomes is equal across all demographic groups. 
Using the notations above, the learning problem of a model $f$ under demographic parity constraints can be expressed as follows:
\begin{equation} \begin{aligned} \argmin_{w_f} \quad & \mathbb{E}_{(x,y,s) \sim p}{\;\mathcal{L_Y}(f_{w_f}(x), y)} \\ \textrm{s.t.} \quad & |\mathbb{E}_{(x,y,s) \sim p}{(\hat{f}_{w_f}(x) | s=1)} - \mathbb{E}_{(x,y,s) \sim p}{(\hat{f}_{w_f}(x) | s=0)}| < \epsilon \end{aligned} \label{eq:demographic-parity}
\end{equation}
Where $\hat{f}$ represents the output prediction after threshold (e.g., $\hat{f}_{w_f}(x)=\mathds{1}_{f_{w_f}(x)>0.5}$). 
The parameter $\epsilon$ represents the deviation permitted from perfect statistical parity, allowing for flexibility in balancing accuracy and fairness. 
In the following, this deviation is noted as  \emph{Disparate Impact} (DI),  representing the absolute difference in positive outcomes between the two demographic groups.



Although numerous methods exist to solve the problem described in Equation~\ref{eq:demographic-parity}, we focus in this work on the family of fair adversarial learning, which has been shown to be the most powerful framework for settings where acting on the training process is an option (i.e., in-processing method) ~\citep{zhang2018,adel2019one,grari2022adversarial}.
One of the most well-known fair adversarial approaches by~\citet{zhang2018} 
is framed as follows: 
\begin{equation}
\begin{aligned}
    \min_{w_f} \quad & {\mathbb{E}_{(x,y,s) \sim p}{\;\mathcal{L_Y}(f_{w_f}(x), y)}} \\
    \textrm{s.t.} \quad & \min_{w_g} \mathbb{E}_{(x,y,s) \sim p}{\mathcal{L_S}(g_{w_g}(f_{w_f}(x)), s)} > \epsilon'
\end{aligned}
    \label{eq:global-fairness-adv}
\end{equation}
Where $\mathcal{L}_\mathcal{S}$  represents a  loss for sensitive reconstruction (e.g. 
a log loss for a binary sensitive attribute). In this adversarial formulation, the goal is to learn a model~$f$ that minimizes the traditional loss of the predictor model, while simultaneously ensuring that an adversary~$g$ with parameters~$w_g$ cannot effectively distinguish between the two sensitive demographic groups based on the predictor's output~$f_{w_f}(x)$. 
The fairness constraint is thus imposed here as the adversary's ability to reconstruct the sensitive attribute, which should be limited, i.e. the value of the loss function $\mathcal{L_S}(g_{w_g}(f_{w_f}(x)), s)$ should be above a minimum value $\epsilon'$. 
In practice, to achieve a balance between the predictor's and the adversary's performance, a relaxed formulation of Equation~\ref{eq:global-fairness-adv} is used:
$\min_{w_f}\max_{w_g} \quad {\mathbb{E}_{(x,y,s) \sim p}{\mathcal{L_Y}(f_{w_f}(x), y)}} - \lambda \mathbb{E}_{(x,y,s) \sim p}{\mathcal{L_S}(g_{w_g}(f_{w_f}(x)), s)}$.
The coefficient $\lambda\in \mathbb{R}^+$ controls the trade-off between the predictor's performance on the task of predicting~$Y$ and the adversary's performance on reconstructing the sensitive attribute. A larger value of $\lambda$ emphasizes the importance of restricting the adversary's ability to reconstruct the sensitive attribute, while a smaller value prioritizes the performance of the predictor on the main task.


\emph{Equalized Odds}: Equalized Odds (EO)~\citep{hardt2016} is another group fairness criterion that requires the classifier to have equal true positive rates (TPR) and false positive rates (FPR) across demographic groups. 
This criterion is especially relevant when misclassification can have significant impacts on individuals from different groups. 
To achieve EO, \cite{zhang2018} employs an adversarial learning approach by concatenating  
the true outcome $Y$ to the input of the adversary. 

    \begin{figure}
        \centering
        \includegraphics[width=0.40\linewidth]{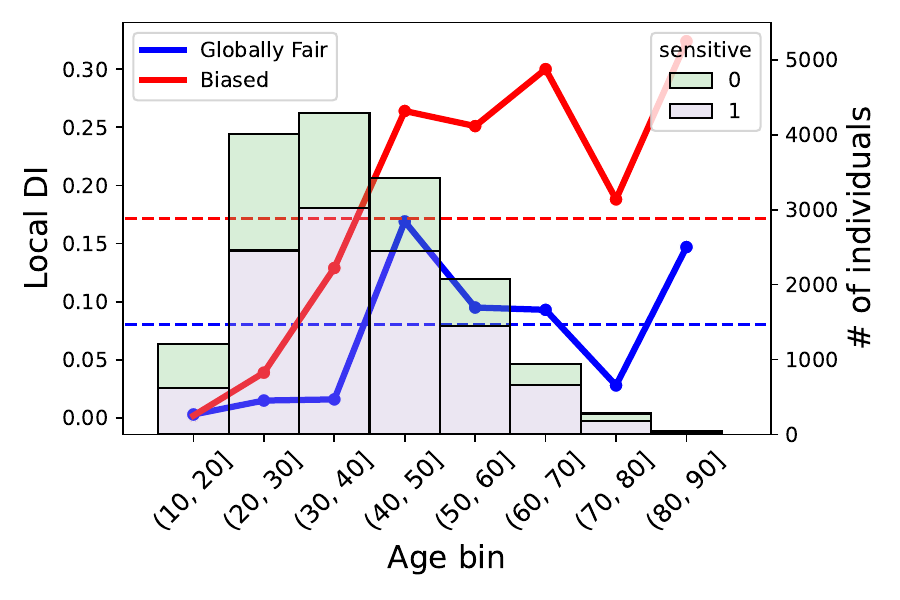}
        \includegraphics[width=0.40\linewidth]{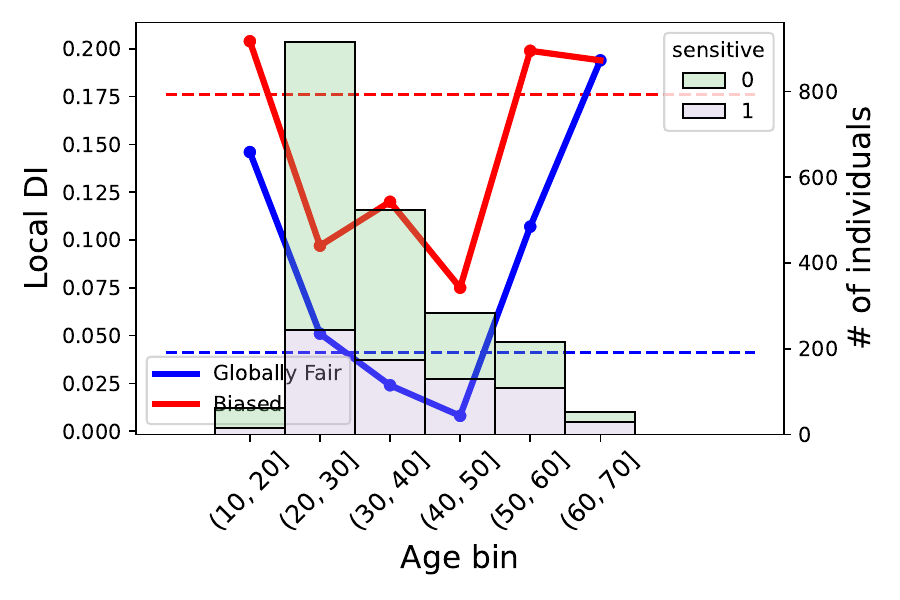}
        \caption{Visualizing local fairness: DI values measured globally (dashed lines) and locally (full lines) for a biased model (red) and one trained with a DI objective (in blue) (\cite{zhang2018} with  
$\lambda$=3). The bar plots show the sensitive attribute distribution in each age group. Left: Adult dataset; Right: Compas.}
        \label{fig:local-fairness-pb}
    \end{figure}


\subsection{The Local Fairness problem}
The global aspect of these group fairness criteria begs the question of the emergence of local undesired behaviors: by enforcing constraints on global averages between sensitive groups, we still expect that some local differences may persist~\citep{krco2023mitigating}.
We illustrate this phenomenon through a simple experiment, shown in Fig.~\ref{fig:local-fairness-pb}. On two datasets, Adult and Compas (described in App.~\ref{sec:app-datasets}), two models are trained: an unconstrained model solely optimizing for accuracy (called \emph{Biased}, in red), and the adversarial model from~\cite{zhang2018} (in blue) optimizing for Demographic Parity for the sensitive attributes \emph{gender} (Adult) and \emph{race} (Compas).
For each model, two types of Disparate Impact (DI) values are shown: the global DI values, calculated over all the test set (dashed lines); and the local ones, calculated in subgroups of the population (full lines). The subgroups are defined here as \emph{age categories}: discretized bins of the continuous attribute \emph{age}. Although local DI values are generally lower for the fair model, they vary a lot across subgroups, sometimes remaining unexpectedly high. This is especially true for less populated segments (e.g., higher \emph{age} values), and segments where the sensitive attribute distribution is extremely unbalanced: as the fairness constraint only concerns global averages, more attention is put on densely populated regions. On the other hand, less populated segments are more likely to be ignored during the training.

These local differences echo the long-asserted claim that the blunt application of group fairness metrics bears inherent inequalities through their failure to account for any additional context~\citep{selbst2019fairness,binns2020apparent}. Here, although reductive, the additional context we refer to is the information already available in the dataset~$X$, in which \emph{comparable} subgroups~\citep{wachter2021why} can be drawn to evaluate fairness.  
This helps defining the notion of \emph{Local Fairness} that is the focus of this paper: a locally fair model thus guarantees minimal differences in expectations within these comparable subgroups of~$X$. 
Contrary to works on intersectional fairness~\citep{kearns18}, the desired behavior in Fig.~\ref{fig:local-fairness-pb} is thus not to treat \emph{age} as a sensitive attribute: predictions $\hat{f}(x)$ are expected to vary along \emph{age}. However, in the Compas dataset for instance, equality between \emph{race} groups is expected to hold regardless of the \emph{age} category considered.
It is important to note that the notion studied here is also different from the one of \emph{individual fairness}, which aims to treat similarly individuals who are close w.r.t. some predefined similarity measure (see, e.g.
~\cite{dwork2012}), without any notion of sensitive data,  rather than minimize DI among subgroups of individuals.

Having knowledge of these subgroups 
at training time would mean that it could be included as an additional constraint in the learning objective, akin to the work of~\cite{vzliobaite2011handling}. The criterion they propose, Conditional Demographic Disparity, measures Demographic Disparity across user-defined subcategories.
However, several issues make this difficult, if not impossible, in practice. Besides that such expert knowledge is generally unavailable, or costly to acquire, the subgroups definitions might even be inconsistent across different testing environments (e.g. conflicting legal definitions of job categories or gender~\citep{wachter2021why}), making its optimization futile.
Furthermore,  exploring multiple categories is problematic in a combinatorial perspective.
In contrast, \cite{hashimoto2018fairness} considers the case of unknown subgroups via the Distributionally Robust Optimization framework, but do not attempt to enforce fairness constraints. 
In this paper, we propose to optimize accuracy while adhering to a worst-case fairness constraint, an objective that was originally introduced to enhance fairness generalization capabilities in scenarios involving distribution drift or noisy labels (cf. Sec.~\ref{sec:background-fairdro}). 
We define the subpopulations of interest, for which we aim to optimize fairness, as distributions $q$ within an \emph{uncertainty set} $\mathcal{Q}$,
and present the DRO framework for the Demographic Parity criterion as follows: 
\begin{equation} \begin{aligned} \min_{w_f} \quad & \mathbb{E}_{(x,y,s) \sim p}{\;\mathcal{L}_{\mathcal{Y}}(f_{w_f}(x), y)} \\
 \textrm{s.t.} \quad & \max_{q \in \mathcal{Q}} {\left|\mathbb{E}_{(x,y,s) \sim q}{\left[\hat{f}_{w_f}(x) | s=1\right]} - \mathbb{E}_{(x,y,s) \sim q}{\left[\hat{f}_{w_f}(x) | s=0\right]}\right|} < \epsilon \end{aligned} \label{eq:local-fairness} \end{equation}
The constraint ensures that the Disparate Impact remains less than a predefined threshold $\epsilon$ under the worst-case distribution
 $q \in \mathcal{Q}$. 
Working with distribution $q$ allows us to enforce local fairness by targeting subpopulations of interest, thus creating a more focused and adaptable model that  addresses fairness problems both globally and at a granular level. 


\subsection{Related Works}
\label{sec:background-fairdro}

Several works have proposed to address the objective in Eq~\ref{eq:local-fairness}, either to ensure better fairness generalization capabilities in drift scenarios~\citep{rezaei2021robust,ferry2022improving,wang2023robust} or when facing noisy labels~\citep{mandal2020ensuring,wang2020robust,roh2021sample}. 
The uncertainty set~$\mathcal{Q}$ then represents the perturbations that might affect the data at test time, and can therefore take several forms. For instance, \cite{mandal2020ensuring} does not impose any constraint on the distribution shape of $q \in \mathcal{Q}$, only setting maximal perturbations on probabilities of training samples. In this case, the objective is to guarantee fairness in the worst-possible
perturbations (not necessarily representing a specific drift). 
On the other hand, the uncertainty set $\mathcal{Q}$ is commonly defined as a ball centered on $p$ using distribution distances or similarities. Examples include 
Total Variation distance~\citep{wang2020robust}, Wasserstein distance~\citep{wang2021wasserstein} or Jaccard index~\citep{ferry2022improving}.

Nonetheless, due to the discrete nature of the problem expressed in Eq~\ref{eq:local-fairness} (the constraint is applied on $\hat{f}$ which is binary),  most existing works restrict to linear models~\citep{wang2020robust,rezaei2020fairness,mandal2020ensuring,taskesen2020distributionally}, or rule-based systems~\citep{ferry2022improving}. This allows them to look for analytical solutions using linear programming. Although~\cite{rezaei2021robust} is an exception in this regard, they suffer from several drawbacks, namely requiring knowledge about the target distribution at train time and about the sensitive attribute at test time.
Solving Equation~\ref{eq:local-fairness} using a wider class of models remains therefore, to the best of our knowledge, unexplored.

\section{ROAD: Robust Optimization for Adversarial Debiasing} 
\label{sec:proposition-road}
\subsection{Formalization}

To overcome the limitations of previous works, we introduce our proposition 
to address the fairness generalization problem
by combining adversarial optimization and the DRO framework. 
In order to learn a predictor ${f}_{w_f}$ 
that is fair both globally and for any subregion of the feature space, the idea is therefore to boost, at each optimization step, the importance of regions $q$ for which the sensitive reconstruction is the easiest for an optimal adversary $g_{w_g*}$ given the current prediction outcomes.
Rewriting the fairness constraint of Equation~\ref{eq:local-fairness} with an adversary~$g_{w_g}:\mathcal{Y}\rightarrow \mathcal{S}$, we thus focus on the following problem 
for Demographic Parity\footnote{Adapting our work to EO is straightforward: as described in Sec.~\ref{sec:background-groupfairness}, adapting the adversarial method of~\cite{zhang2018} to the EO task simply  
requires to concatenate the 
 the true outcome $Y$ to the prediction $f(x)$ as input of the adversarial classifier. The same process can be followed for ROAD.
}:
\begin{equation}
\begin{aligned}
\min_{w_f} \quad & {\mathbb{E}_{(x,y,s) \sim p}{\;\mathcal{L_Y}(f_{w_f}(x), y)}} \\
\textrm{s.t.} \quad & \min_{q \in \mathcal{Q}}{\mathbb{E}_{(x,y,s) \sim q}{\mathcal{L_S}(g_{{w_g}^{*}}(f_{w_f}(x), s))}} > \epsilon'
\\
\text{with }  &{w_g}^{*} = \arg\min_{w_g} \mathbb{E}_{(x,y,s) \sim p}{\mathcal{L_S}(g_{w_g}(f_{w_f}(x), s))}
\end{aligned}
\label{eq:fair-dro0}\end{equation}
A major challenge with this formulation is that exploring all possible distributions~$\mathcal{Q}$ is infeasible in the general sense. 
Worse, modeling distribution~$q$ directly over the whole feature space as support is very difficult, and usually highly inefficient, even for~$\mathcal{Q}$ restricted to distributions close to~$p$. This motivates an adversarial alternative, which relies on  importance weighting of training samples from~$p$.  
We therefore restrict~$\mathcal{Q}$ to the set of distributions that are absolutely continuous with respect to $p$\footnote{In the situation where all distributions in $Q$ are absolutely continuous with respect to $p$ all measurable subset $A \subset X \times Y$, all $q \in \mathcal{Q}$, $q(A) > 0$ only if $p(A) > 0$)}, inspired by~\cite{michel2022distributionally}. This
allows us to write $q = rp$ , with ${r:\mathcal{X} \times \mathcal{S} \rightarrow \mathbb{R}^+}$  a 
function that acts as a weighting factor.
Given a training set $\Gamma$ sampled from $p$, we can thus reformulate the overall objective, by
substituting $q$ with $rp$ and applying its Lagrangian relaxation, as an optimization problem on $r \in {\mathcal{R} = \{r \,|\, rp \in \mathcal{Q}\}}$:
\begin{align}
&\min_{w_f} \max_{r \in \mathcal{R}} 
\frac{1}{n} \sum_{i=1}^{n} \mathcal{L_Y}(f_{w_f}(x_i), y_i) - \lambda_g  
\frac{1}{n} \sum_{i=1}^{n} {r(x_i,s_i) 
 \mathcal{L_S}(g_{w_g^*}(f_{w_f}(x_i)), s_i)} 
 \\
\text{with } &w_g^* = \argmin_{w_g} 
\frac{1}{n} \sum_{i=1}^{n} \mathcal{L_S}(g_{w_g}(f_{w_f}(x_i)), s_i) 
 \nonumber
\end{align}
With $\lambda_g$ a regularization parameter controlling the trade-off between accuracy and fairness in the predictor model. 
In the following, we describe two constraints, inspired from the DRO literature, that we consider to ensure $q$ keeps the properties of a distribution and avoid pessimistic solutions.

\paragraph{Validity Constraint} 

To ensure~$q$ keeps the properties of a distribution (i.e., $r \in {\cal R}$), previous works in 
DRO (e.g.~\cite{michel2022distributionally}) enforce the constraint $\mathbb{E}_{(x,s)\sim p}  r(x, s) = 1$ during the optimization. 
In the context of fairness, we argue that this constraint is not sufficient to ensure a safe behavior with regard to the fairness criterion, as it allows disturbances in the prior probabilities of the sensitive (i.e., $q(s) \neq p(s)$).  As discussed more deeply in Appendix~\ref{theoanalysis}, 
this 
may lead to a shift of the optimum of the problem, by inducing a stronger mitigation emphasis on samples from the most populated demographic-sensitive group. 
To avoid this issue, we propose to further constrain~$r$ by considering a restricted set $\tilde{\cal R}=\{r \in {\cal R}\,|\, rp \in \tilde{\cal Q}\}$, with $\tilde{\cal Q} \subset {\cal Q}$ such that: $\forall s, q(s)=p(s)$. 
To achieve this, we rely on the following constraint: $\forall s, \mathbb{E}_{p(x|s)} \ r(x,s)=1$.
Besides guaranteeing the desired property $q(s)=p(s)$ (proof in Sec.~\ref{proof_eqprior}), we also note that ensuring these constraints still imply the former one: $\mathbb{E}_{p(x,s)}  r(x, s) = 1$, which guarantees that $q(x,s)$ integrates to $1$ on its support.
We further discuss 
the benefits of this conditional constraint in Section~\ref{empanalysis}.

\paragraph{Shape Constraint}
As discussed in Section~\ref{sec:background-fairdro}, the definition of $\mathcal{Q}$ heavily impacts the desired behavior of the solution. In particular, controlling the shape of the allowed distributions~$q$ is especially crucial in a setting such as ours, where the focus of the mitigation process is done dynamically.
Without any constraint (as proposed by~\cite{mandal2020ensuring}), the mitigation could indeed end up focusing 
on specific points of the dataset where the sensitive reconstruction from $f_{w_f}(X)$ is the easiest, using very sharp distributions~$q$ close to a Dirac.
This may turn particularly unstable and, more critically, could concentrate the majority of fairness efforts on a relatively small subset of samples.
To control the shape of the bias mitigation distribution $q$, 
%
we therefore choose to consider $\cal Q$ as a
KL-divergence ball centered on the training distribution~$p$.
However, similarly to \cite{michel2022distributionally}, 
we do not explicitly enforce the
KL constraint (due to the difficulty of projecting onto the KL ball) and instead use a relaxed form. Using previous notations, the KL constraint takes the simple form $\text{KL}(q || p) = \text{KL}(pr || p) = \mathbb{E}_{p} r \log \frac{pr}{p} = \mathbb{E}_{p} \,r \log r$. 

The spread of $q$ can then be controlled with a temperature weight~$\tau$ in the overall optimization process. 
Setting $\tau=0$ means that no constraint on the distribution of~$r$ is enforced, thus encouraging $r$ to put extreme attention to lower values of $\mathcal{L}_S$. 
On the other hand, higher values of $\tau$ favors distributions $q$ that evenly spreads over the whole dataset, 
hence converging towards a classical globally fair model for highest values (cf. Section \ref{sec:background-groupfairness}).
\paragraph{ROAD Formulation} 
The overall optimization problem of our Robust Optimization for Adversarial Debiasing (ROAD) framework can thus finally be formulated as (full derivation given in \ref{derivationROAD}):
\begin{eqnarray}\label{eq:final-pb}
\min_{w_f} \max_{\substack{r \in 
\tilde{\cal R}}}
\frac{1}{n} \sum_{i=1}^{n} \mathcal{L}_{Y}(f_{w_f}(x_i), y_i) - \lambda_g [ \frac{1}{n} \sum_{i=1}^{n}{({r}(x_i,s_i)\mathcal{L}_{S}(g_{w_g^*}(f_{w_f}(x_i)), s_i)))}  \nonumber \\ + \tau \underbrace{\frac{1}{n} \sum_{i=1}^{n}{({r}(x_i,s_i)\log({r}(x_i,s_i))}}_{\text{KL constraint}}]  \\
\text{with } w_g^* = \argmin_{w_g} \frac{1}{n} \sum_{i=1}^{n} \mathcal{L}_{S}(g_{w_g}(f_{w_f}(x_i)), s_i) \nonumber 
\end{eqnarray}
In the next section, we discuss two implementations of this approach.

\subsection{Two Implementations}
\subsubsection{BROAD: A Non-Parametric Approach}
\label{broad}
Let us first introduce a non-parametric approach,
called Boltzmann Robust Optimization Adversarial Debiasing (BROAD), where each $r(x_i,s_i)$ value results from the inner maximization problem from Eq.\ref{eq:final-pb}. As described below, this inner optimization accepts an analytical solution, whenever $r$ values respect the aforementioned  
conditional  
validity constraints 
(proof in Appendix~\ref{sec:app-proof-BROAD}).

\begin{lemma} 
\label{BROAD_lemma}
(Optimal Non-parametric Ratio) 
Given a classifier $f_{w_f}$ and an adversary $g_{w_g}$, the optimal weight $r(x_i,s_i)$ for any sample from the training set, is given by:\\ 
$$ {r}(x_i,s_i) = \frac{\e^{-\mathcal{L}_{S}(g_{w_g}(f_{w_f}(x_i)), s_i)/\tau}}{\frac{1}{n_{s_i}}\sum_{(x_j,s_j)\in \Gamma, s_j=s_i} 
\e^{-\mathcal{L}_{S}(g_{w_g}(f_{w_f}(x_j)), s_j)/\tau}}
$$

\end{lemma}

With 
$n_{s_i}=\sum_{i=1}^{n}\mathbbm{1}_{s=s_i}$. This expression allows us to set optimal weights for any sample from the training dataset, at no additional computational cost compared to a classical adversarial fairness approach such as \cite{zhang2018}. 
However, this may induce an unstable optimization process, since weights may vary abruptly for even very slight variations of the classifier outputs. 
Moreover, it implies individuals weights, only interlinked via the outputs from the classifier, hence at the risk of conflicting with our notion of local fairness. 
We therefore propose another -  parametric - implementation, described in the next section, that improves the process by introducing local smoothness in the fairness weights.

\subsubsection{Parametric Approach}
\label{sec:param_app}

To introduce more local smoothness in the fairness weights assigned to training samples, we propose an implementation of the $r$ function via a neural network architecture.  
Our goal is to ensure that groups of similar individuals, who might be neglected in the context of group fairness mitigation (e.g., due to their under-representation in the training population, cf. Fig.~~\ref{fig:local-fairness-pb}),   receive a similar level of attention 
during the training process. However, solely relying  on adversarial accuracy, 
as done in BROAD, may induce many irregularities in such groups. 
The 
lipschitzness of neural networks can add additional implicit locality smoothness assumptions in the input space, thus helping define the distributions $q$ as subregions of the feature space. 
Note that, in this approach, the network architecture therefore plays a crucial role in how local the behavior of $r_{w_r}$ will be: more complex networks will indeed tend to favor more local solutions, for a same value of~$\tau$. In particular, a network of infinite capacity that completes training will have, in theory, the same behavior as BROAD.

To enforce the conditional validity constraint presented earlier, we employ an exponential parametrization with two batch-level normalizations, 
 one for each demographic group. For each sample ($x_i
, y_i,s_i$) in the mini-batch, we define the normalized ratio as: 
$$\forall i,  {r_{w_r}}(x_i,s_i) = \frac{\e^{h_{w_r}(x_i,s_i)}}{\frac{1}{n_{s_i}}\sum_{(x_j,s_j)\in \Gamma, s_j=s_i}\e^{h_{w_r}(x_j,s_j)}}
$$
with $h:{\cal X} \times \{0;1\} \rightarrow \mathcal{R}$ a neural network with weights ${w_r}$.   

To train ROAD, we use an iterative optimization process, alternating between updating the predictor model's parameters $w_f$ and updating the adversarial models' parameters $w_g$ and $w_r$ by multiple steps of gradient descent. 
This leads to a far more stable learning process and 
prevents the predictor classifier from dominating the adversaries. More details are provided in the appendix (see Alg.~\ref{alg:ROAD}). 
\section{Experiments}
\label{sec:experiments}
The experimental evaluation of our approaches is three-fold: First, we assess how effectively ROAD and BROAD ensure Local Fairness in unknown subgroups. Then, we focus on fairness generalization in the face of distribution shift. 
Finally, we conclude the analysis with several ablation studies, allowing a better understanding of the proposed methods.

\subsection{Assessing Local Fairness }
\label{sec:experiments-localfairness}

In this first experiment, we assess how effective ROAD is for generating predictions that are locally fair for unknown subpopulations, while guaranteeing a certain level of global accuracy and global fairness.
For this purpose, we use 3 different popular data sets often used in fair classification, chosen for their diverse characteristics and relevance to various real-world scenarios and described in Appendix~\ref{sec:app-datasets}: Compas~\citep{angwin2016machine}, Law~\citep{wightman1998lsac} and German Credit~\citep{german_credit_data}. Each dataset is split into training and test subsets, and the models described below are trained to optimize accuracy while mitigating fairness with respect to a sensitive attribute~$S$. 

To assess fairness at a local level, various subpopulations chosen among features of $X$, 
i.e. excluding $S$, are selected in the test set. As an example on the Compas dataset, in which the sensitive attribute is \emph{Race}: to create the subgroups, the \emph{Age} feature is discretized into buckets with a 10-year range. These intervals are then combined with the \emph{Gender} feature, identifying 12 distinct subgroups. 
As measuring DI in segments of low population is highly volatile, we filter out subgroups containing less than $50$ individuals (more details in Appendix~\ref{sec:app-subgroups-description}). These subgroups are unknown at training time, and chosen arbitrarily to reflect possible important demographic subgroups (see Section~\ref{sec:experiments-ablation-subgroups} for further discussion). 
Given the set of these subgroups~$\mathcal{G}$, the local fairness is then assessed 
on the worst Disparate Impact value among these subgroups: $\emph{Worst-1-DI}=\max_{g \in \mathcal{G}}|\mathbb{E}_{(x,s) \in g}{(\hat{f}_{w_f}(x) | s=1)} - \mathbb{E}_{(x,s) \in g}{(\hat{f}_{w_f}(x) | s=0)|}$.

To evaluate our approach, we compare our results with the globally fair adversarial models from~\cite{zhang2018} and~\cite{adel2019one}, as well as 3 existing works that aim to address fairness generalization: FairLR~\citep{rezaei2020fairness}, Robust FairCORELS~\citep{ferry2022improving} and CUMA~\citep{wang2023robust} (cf. Appendix~\ref{sec:app-competitors} for more details). 

As local fairness can only be measured against 
global accuracy and global fairness, we evaluate the approaches by plotting the tradeoffs between global accuracy and worst-1-DI  subject to a global DI constraint (we choose $DI\leq0.05$, a common threshold in the fairness literature~\citep{pannekoek2021investigating}).
To ensure a thorough exploration of these tradeoffs, we sweep across hyperparameter values for each algorithm (details of hyperparameter grids in App.~\ref{sec:app-hyperparameters-architecture}). 
Figure~\ref{fig:results-localfairness} shows the resulting Accuracy-Worst-1-DI Pareto curves for each method. For all datasets, ROAD mostly outperforms all other methods. This tends to show how the proposed method efficiently maximizes local fairness, without sacrificing any other desirable criterion too much. On the other hand, BROAD does not always perform as effectively as  ROAD, illustrating the benefit resulting from the local smoothness induced by the use of a neural network. 
Interestingly, despite not including any robustness component, globally fair methods of~\cite{zhang2018} and~\cite{adel2019one} still manage 
to slightly reduce local bias through their global mechanisms. 

\begin{figure}
    \centering
    \includegraphics[width=0.32\linewidth]{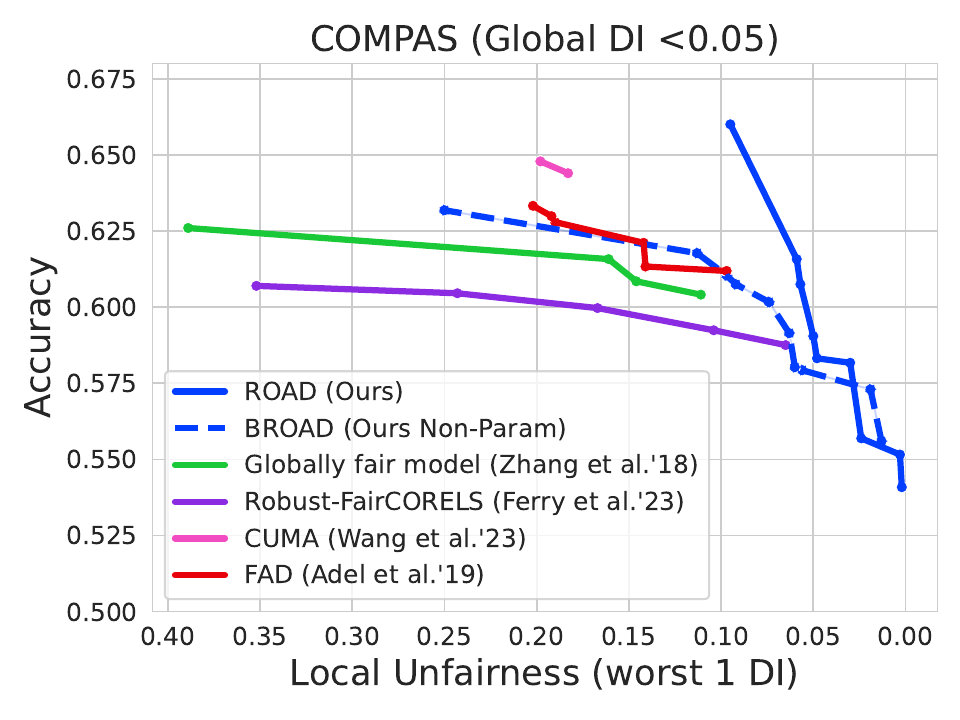}
    \includegraphics[width=0.32\linewidth]{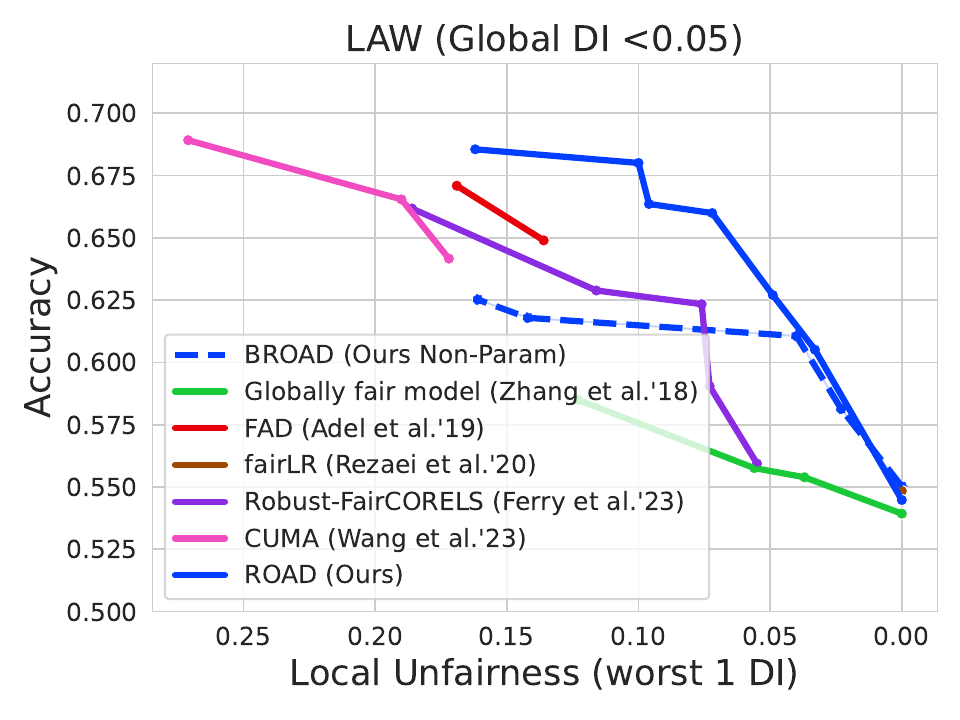}
    \includegraphics[width=0.32\linewidth]{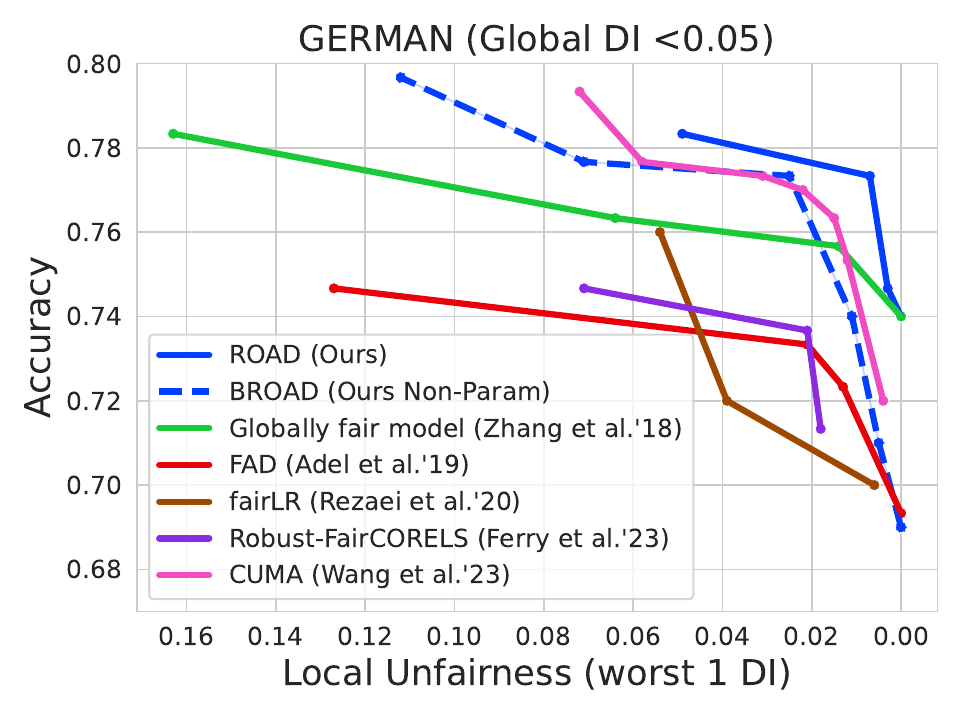}
    \caption{Results for the experiment on Local Fairness. For all datasets, the X-axis is Worst-1-DI, Y-axis is Global accuracy. The curves represented are, for each method, the Pareto front for the results satisfying the imposed global fairness constraint (here, Global DI $< 0.05$ for all datasets). 
    }
    \label{fig:results-localfairness}
\end{figure}




\subsection{Experiments on Distribution Drift}
\label{sec:experiments-drift}

As discussed in Section~\ref{sec:background-fairdro}, DRO-based techniques have been considered before to help with the generalization of fairness. In this section, we therefore aim to show how our approach also leads to a better generalization of fairness in the face of distribution shift in addition to better-protecting subpopulations. For this purpose, we replicate the experimental protocol of~\cite{wang2023robust}: after training classifiers on the training set of the classical Adult dataset (1994), we evaluate the tradeoff between accuracy and global fairness (measured with Equalized Odds (EO)) on the 2014 and 2015 Folktables datasets~\citep{ding2021retiring}, containing US Census data from corresponding years, thus simulating real-world temporal drift.
The same approaches as in the previous section, adapted to optimize for EO (details in Appendix~\ref{sec:app-competitors}), are tested. Once again, the hyperparameters of every method are adjusted to maximize the two considered criteria, and the Pareto front is shown in Fig.~\ref{fig:adult-drift}.

Results on the classical Adult test set (in-distribution, left figure) are somewhat similar for most methods, with CUMA~\citep{wang2023robust} slightly out-performing other methods. However, on drifted test sets (center and right figures), ROAD seems to achieve significantly better results than other methods, including other DRO-based fairness approaches. This suggests that the parametric implementation 
proposed in the paper is better suited to ensure robust behavior.

\begin{figure}
   \centering
   \includegraphics[width=0.32\linewidth]{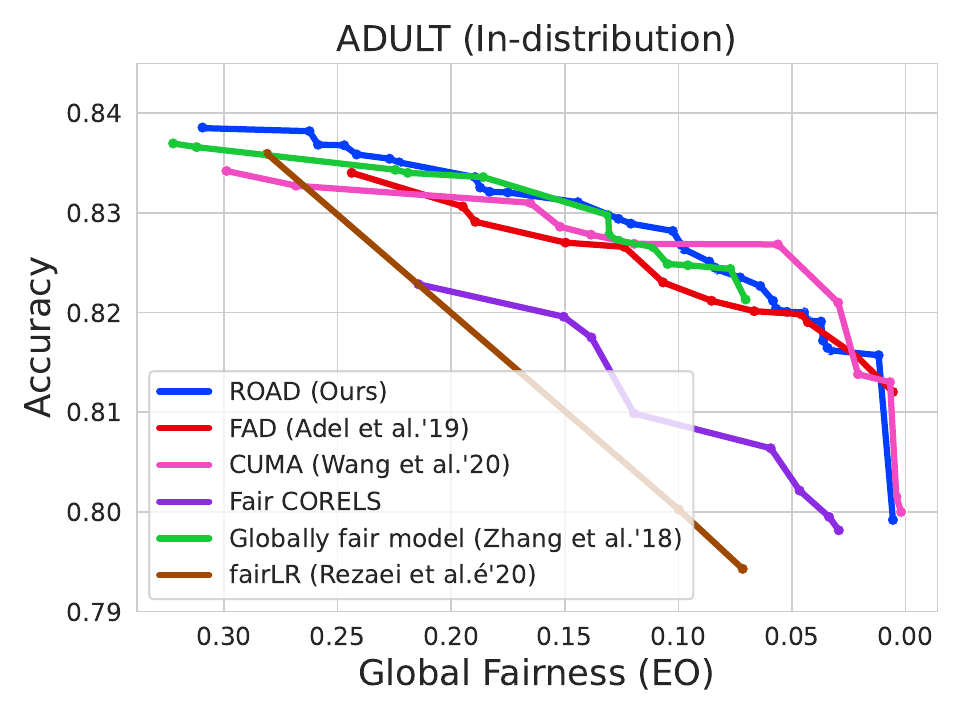}
   \includegraphics[width=0.32\linewidth]{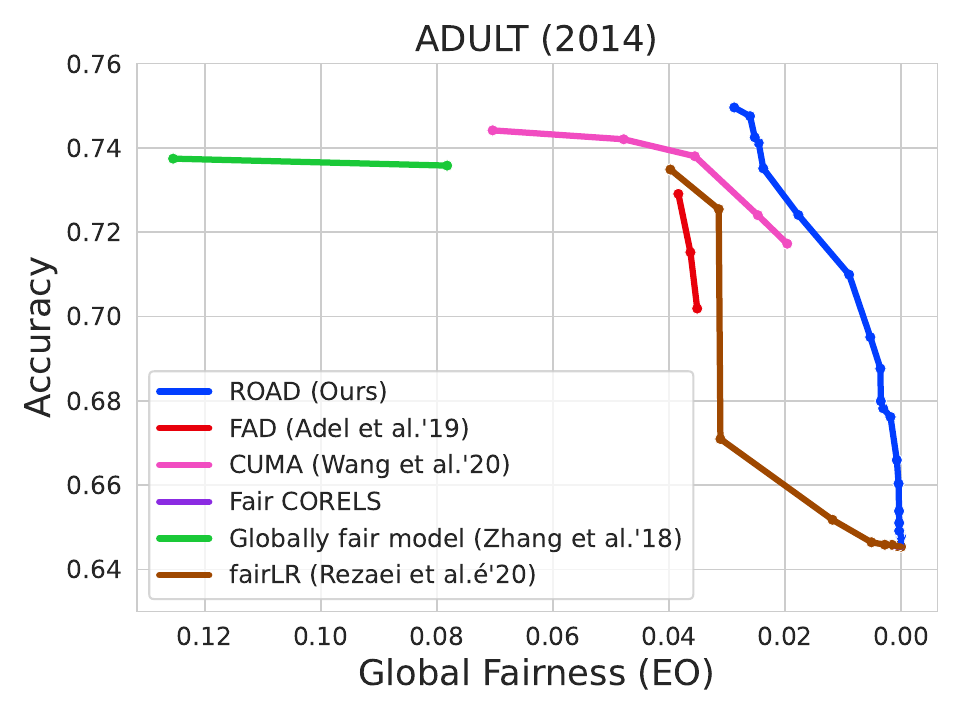}
 \includegraphics[width=0.32\linewidth]{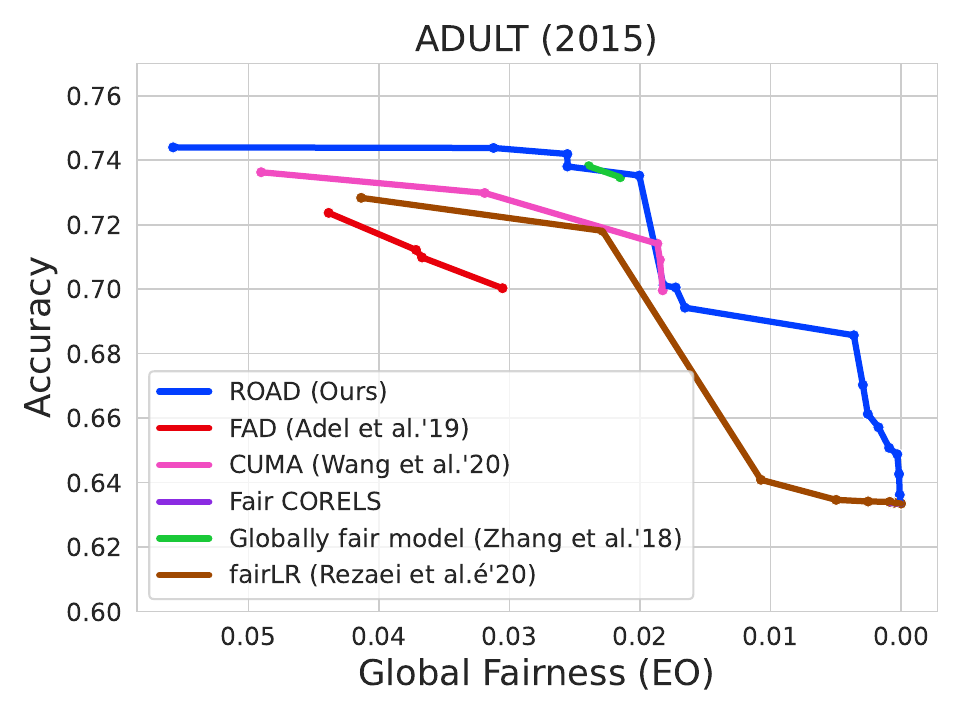}
  \caption{Pareto front results 
  on distribution drift
  using the Adult dataset. For all figures, the X-axis is Equalized Odds; the Y-axis is Accuracy. 
  Left: in-distribution (i.e. Adult UCI in 1994) test dataset; Center and Right: resp. 2014 and 2015 test datasets from Folktables~\citep{ding2021retiring}.}
  \label{fig:adult-drift}
\end{figure}

\subsection{Ablation studies}
\label{sec:experiments-ablation}


\subsubsection{Does the adversary focus on the right subpopulations? How does the temperature parameter help?}
\label{sec:experiments-ablation-temperature}

\begin{figure}
    \centering\includegraphics[width=0.3\linewidth]{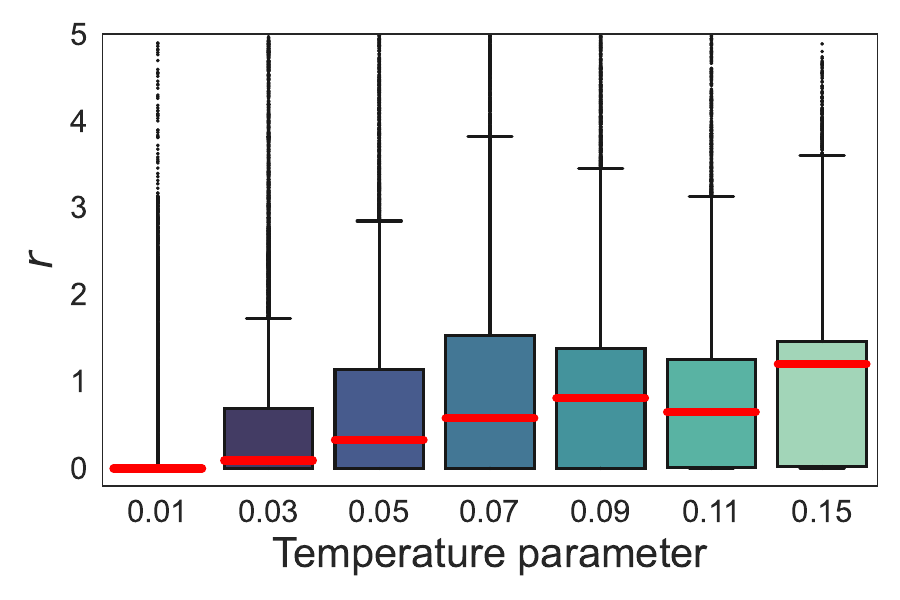}
    \includegraphics[width=0.3\linewidth]{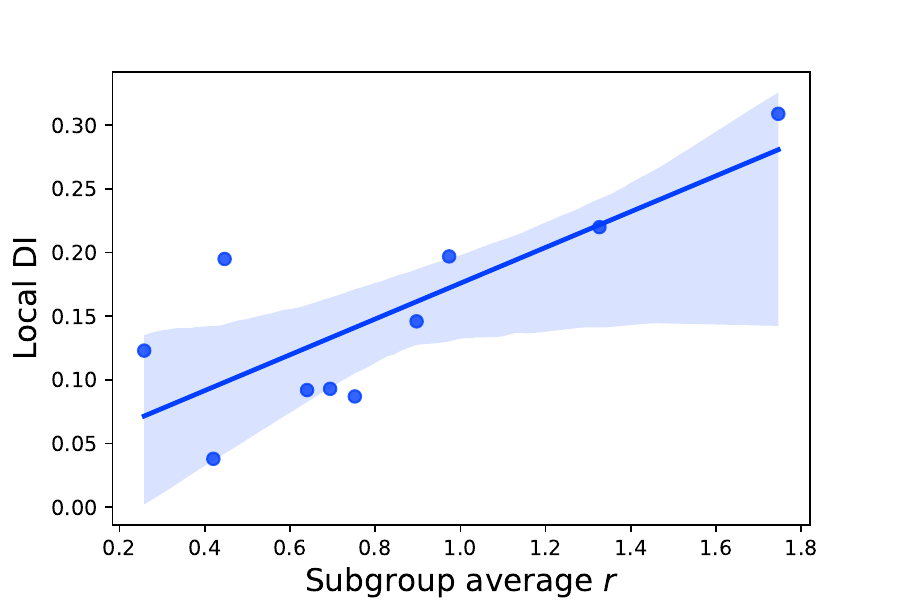}
    \includegraphics[width=0.25\linewidth]{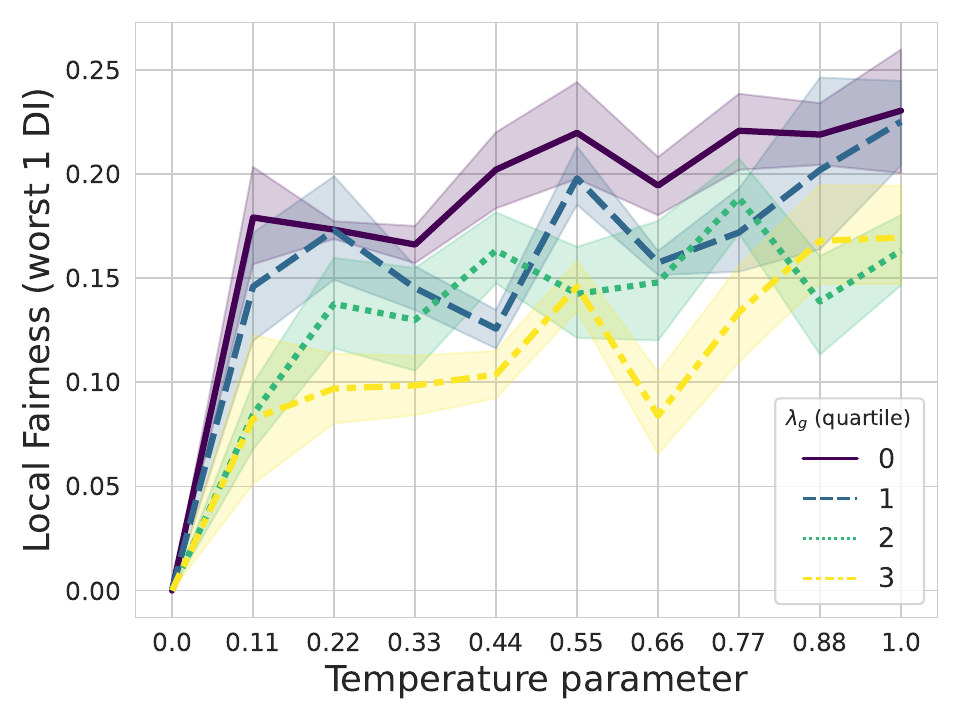}
    \caption{Analysis of the behavior of ROAD on Compas. Left: distribution of $r$ for several values of $\tau$ at epoch 200 (truncated at $r>5$). Center: Relationship between Local DI and the average value of $r$ assigned to instances belonging to the corresponding subgroups. Each dot is a subgroup. Right: Worst-1-DI as a function of $\tau$ for different values for $\lambda_g$ (quartiles between $0.0$ and $10.0$)}
    \label{fig:ablation-temperature}
\end{figure}

The behavior of ROAD depends on 
a temperature hyperparameter~$\tau$, which controls the extent to which the distributions~$q\in \mathcal{Q}$ are allowed to diverge from the training distribution $p$. 
The impact of $\tau$ 
can be observed in the left figure of Fig.~\ref{fig:ablation-temperature} for the Compas dataset. As values of $\tau$ increase, the variance of the distribution of $r$ decreases, going from having most weights close to $0$ and very high importance on a few others, to having most weights $r_i$ lying around $1$.
Choosing the right value of $\tau$ thus helps control the emphasis put on some subpopulations. 

A critical assumption ROAD relies on is that the adversary~$r$ puts more attention on locally unfair regions.
We test this assumption on the Compas dataset (same subgroups as in Section~\ref{sec:experiments-localfairness}) and observe the results in the center figure of Fig.~\ref{fig:ablation-temperature}. For each subgroup $k \in \mathcal{G}$ (blue dots), we measure its local fairness (y-axis) and the average weight $\mathbb{E}_{(x,s) \in k}(r_i(x,s))$ associated to instances of $k$. The graph reveals a correlation between these two notions, suggesting that more emphasis is indeed put on more unfair regions. 
 As a consequence of these two observed results, choosing the value of $\tau$ helps control local bias, as shown in the right picture of Fig.~\ref{fig:ablation-temperature} for various values of $\lambda_g$. 
 The perfect local fairness score achieved when $\tau=0$ is due to a constant model $f_{w_f}$: 
 with no shape constraint imposed, the distribution $r$  concentrates all the fairness effort on each training sample successively, which finally leads to $f(X)=\mathbb{E}[Y]$ for any input. 
 Choosing a greater value of $\tau$ helps regularizing the process, by inducing a distribution $q(x|s)$ closer to $p(x|s)$. 

\subsubsection{How important is the definition of subgroups?}
\label{sec:experiments-ablation-subgroups}

\begin{figure} 
    \centering
    \includegraphics[width=0.95\linewidth]{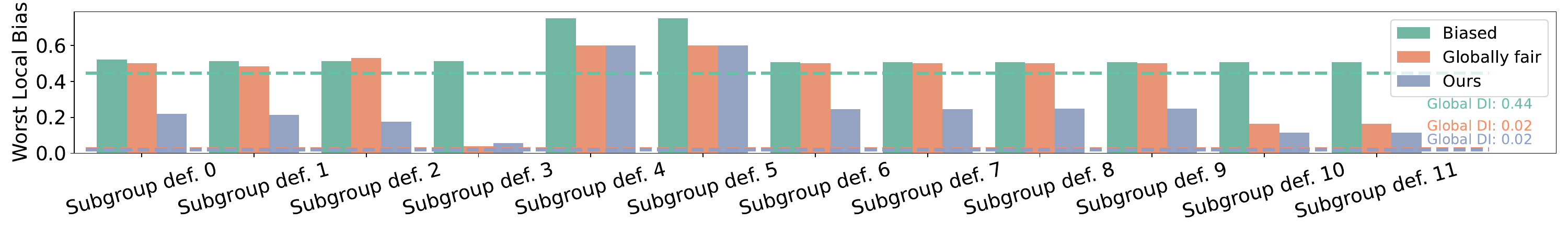}
    \caption{Worst-1-DI scores for subgroups of the Law dataset of various definitions, built by varying \emph{age} bin width and splits along \emph{gender}. Full description or the subgroups is available in Sec.~\ref{sec:app-subgroups-description}
    }
    \label{fig:ablation-subgroups}
\end{figure}

The main motivation for ROAD is its ability to maximize local fairness when the definition of the local subgroups is unknown.
To assess the veracity of this claim, we conduct another experiment where we measure the local fairness of ROAD  
when the definition of these subgroups vary.
Concretely, we train once a biased model, a globally fair model~\citep{zhang2018} and ROAD (with resp. accuracy scores $0.72$, $0.60$, and $0.60$), and measure the local fairness for these models in subgroups of various definitions. These subgroups are defined successively as age bins with a width of 5, 10, 15 and 20, first across the whole population and then across subpopulations of other, non-sensitive, variables. 
Fig.~\ref{fig:ablation-subgroups} shows the local fairness results for the Law dataset (sensitive attribute is \emph{Race}, subgroup attributes are \emph{Age} and \emph{Gender}). As expected, although the worst local DI for ROAD varies when the subgroup definition changes, 
it is almost consistently below the values reached by the globally fair model (except Def.~3 corresponding to the 
largest subgroups). 
This suggests that its tuning is not over-reliant on one subgroup definition, showcasing the flexibility of the approach.

\section{Conclusion}


In this work, we introduced a new problem based on enforcing local fairness in unknown subpopulations.  By leveraging the strengths of adversarial learning and Distributionally Robust Optimization, our proposed framework ROAD provides a powerful approach for this setting, 
address 
the shortcomings of previously proposed DRO-based approaches. Future works include extending our work to settings where the sensitive attribute is not available, and further exploring the optimization of a 3-network adversarial approach. 

\bibliography{iclr2024_conference}
\bibliographystyle{iclr2024_conference}

\newpage
\appendix
\section{Appendix}

This Appendix is composed of seven sections, each providing supplementary information to support the main text. In Section~\ref{derivationROAD}, we detail the derivation of the ROAD objective. In Section~\ref{sec:experiments-ablation-dualcont}, we further discuss the conditional validity constraint introduced in Section~~\ref{sec:proposition-road}, showing analytically and empirically its benefits compared to a global constraint. 
 Our theoretical proof of the non-parametric approach can then be found in Section~\ref{sec:app-proof-BROAD}.
In Section~\ref{App:ablation_studies}, we present additional ablation studies allowing a better understanding of our proposed methods.
Moving forward, Section~\ref{App:Expe_details} outlines the experimental protocols followed in our research, with details on the datasets and competitors considered, while Section~\ref{sec:desc_algo} finally delves into the details of the proposed algorithms.

\subsection{Derivation of ROAD (eq. \ref{eq:final-pb}, with Global and Conditional Normalization)}
\label{derivationROAD}

We start from the initial formulation of the problem as stated in equation \ref{eq:fair-dro0}, that we relax, and include the KL regularization to get 
the quantity $L_{\cal Q}$ to optimize: 

$$L_{\cal Q}\triangleq\mathbb{E}_{p(x,y,s) }\left[{\;\mathcal{L_Y}(f_{w_f}(x), y)}\right] - \lambda_g \left[\mathbb{E}_{q(x,s)}\left[ {\mathcal{L_S}(g_{{w_g}^{*}}(f_{w_f}(x), s))}\right] + \tau \mathbb{E}_{q(x,s)}\left[ \log \frac{q(x,s)}{p(x,s)}\right]\right]$$
where $q \in {\cal Q}$, with ${\cal Q}$ the set of joint distributions over support of $(X,S)$ that are absolutely
continuous  with respect to $p(X,S)$. 

Let us write    $r(x,s)\triangleq\frac{q(x,s)}{p(x,s)}$. Thus, using the importance sampling trick, we get: 
\begin{multline*}
L_{\cal Q}\triangleq\mathbb{E}_{p(x,y,s) }\left[{\;\mathcal{L_Y}(f_{w_f}(x), y)}\right] - \\ \lambda_g \left[\mathbb{E}_{p(x,s)}\left[r(x,s) {\mathcal{L_S}(g_{{w_g}^{*}}(f_{w_f}(x), s))}\right] + \tau \mathbb{E}_{p(x,s)}\left[r(x,s) \log r(x,s)\right]\right]
\end{multline*}

Following this, considering a training dataset $\Gamma=(x_i,y_i,s_i)_{i=1}^n$ sampled from $p(X,Y,S)$, our optimization problem could be given as (which we refer to as single normalization):
\begin{eqnarray*}
\min_{w_f} \max_{\substack{r \in {\cal R}}} \frac{1}{n} \sum_{i=1}^{n} \mathcal{L}_{Y}(f_{w_f}(x_i), y_i) - \lambda_g [ \frac{1}{n} \sum_{i=1}^{n}{({r}(x_i,s_i)(\mathcal{L}_{S}(g_{w_g^*}(f_{w_f}(x_i)), s_i))}  \nonumber \\ + \tau \underbrace{\frac{1}{n} \sum_{i=1}^{n}{({r}(x_i,s_i)\log({r}(x_i,s_i))}}_{\text{KL constraint}}]  \\
\text{with } w_g^* = \argmin_{w_g} \frac{1}{n} \sum_{i=1}^{n} \mathcal{L}_{S}(g_{w_g}(f_{w_f}(x_i)), s_i) \nonumber 
\end{eqnarray*}
where ${\cal R}=\{r|p r \in {\cal Q}\}$ is an uncertainty set ensuring that $q=pr$ is a distribution (i.e., respecting $\mathbb{E}_p(x,s)[r(x,s)]=1$).

As explained in section \ref{sec:proposition-road} (with further analysis in section \ref{theoanalysis}), this set is not restrictive enough to guarantee a stable optimization for our fairness objectives, for settings implying an adversarial trained on a dataset with unbalanced proportions of sensitive values (i.e., $p(S)$ is not uniform). 

To go further, we propose to restrict to $\tilde{\cal Q} \subset {\cal Q}$,  such that any $q \in \tilde{\cal Q}$ respects marginal equalities $q(s)=p(s)$ for any $s \in S$. This leads to: 
\begin{multline*}
    L_{\tilde{\cal Q}}=\mathbb{E}_{p(x,y,s) }\left[{\;\mathcal{L_Y}(f_{w_f}(x), y)}\right] - \\  \lambda_g \left[\mathbb{E}_{p(s)}\mathbb{E}_{q(x|s)}\left[ {\mathcal{L_S}(g_{{w_g}^{*}}(f_{w_f}(x), s))}\right] + \tau \mathbb{E}_{p(s)}\mathbb{E}_{q(x|s)}\left[ \log \frac{q(x,s)}{p(x,s)}\right]\right]
\end{multline*}

From the continuity of distributions in $\cal Q$ w.r.t. $p$ and the fact that $q(s)=p(s)$ for any $q \in \tilde{\cal Q}$, $p(x|s)>0$ whenever $q(x|s)>0$. Thus, using the importance sampling trick and  introducing $r(x|s)\triangleq\frac{q(x|s)}{p(x|s)}$, we get:  
\begin{multline*}
    L_{\tilde{\cal Q}}=\mathbb{E}_{p(x,y,s) }\left[{\;\mathcal{L_Y}(f_{w_f}(x), y)}\right] - \\  \lambda_g \left[\mathbb{E}_{p(s)}\mathbb{E}_{p(x|s)}\left[r(x|s){\mathcal{L_S}(g_{{w_g}^{*}}(f_{w_f}(x), s))}\right] + \tau \mathbb{E}_{p(s)}\mathbb{E}_{p(x|s)}\left[r(x|s) \log r(x|s)\right]\right]
\end{multline*}

Noting that for any $q \in {\tilde{\cal Q}}$, $r(x,s)=r(x|s) r(s)=r(x|s)$, this leads to the following optimization problem, with $\tilde{\cal R}=\{r|p r \in \tilde{\cal Q}\}$: 

\begin{eqnarray*}\label{eq:final-pb2}
\min_{w_f} \max_{\substack{r \in \tilde{\cal R}} 
} \frac{1}{n} \sum_{i=1}^{n} \mathcal{L}_{Y}(f_{w_f}(x_i), y_i) - \lambda_g [ \frac{1}{n} \sum_{i=1}^{n}{({r}(x_i,s_i)(\mathcal{L}_{S}(g_{w_g^*}(f_{w_f}(x_i)), s_i))}  \nonumber \\ + \tau \underbrace{\frac{1}{n} \sum_{i=1}^{n}{({r}(x_i,s_i)\log({r}(x_i,s_i))}}_{\text{KL constraint}}]  \\
\text{with } w_g^* = \argmin_{w_g} \frac{1}{n} \sum_{i=1}^{n} \mathcal{L}_{S}(g_{w_g}(f_{w_f}(x_i)), s_i) \nonumber 
\end{eqnarray*}
which is equivalent to eq \ref{eq:final-pb}, since $\tilde{\cal R}=\{r \in {\cal R}| \forall s, \mathbb{E}_{
p(x|s)} \ r(x, s)= 1\} $. 

\subsection{Why does using two conditional normalization constraints instead of a single global one help?}
\label{sec:experiments-ablation-dualcont}

\subsubsection{Theoretical Implication of the Double  Normalization Constraint}
\label{proof_eqprior}

This section aims at proving that using two conditional normalizations of the $r$ values, one for each sensitive, leads to guarantee $p(s)=q(s)$, as claimed in the main paper and advised in the next sections. 

Assuming we have a function $r$ that respects the property: $\mathbb{E}_{
p(x|s)} \ r(x, s)= 1$ for both $s$, we can first start by observing that the classical global constraint also holds: $\mathbb{E}_{
p(x,s)} \ r(x, s)= \mathbb{E}_{
p(s)} \mathbb{E}_{p(x|s)} \ r(x, s)=  \sum_s 
p(s) = 1$. This allows us to consider $r(x,s)=\frac{q(x,s)}{p(x,s)}$ as a ratio of valid joint distributions, with  $q(x,s)$ which integrates to one over the support of $p$. 

Then, let us analyze the induced marginal $q(s)$ resulting from such implicit distribution: 
$$q(s)=\int q(x,s) dx = \int p(x,s) \frac{q(x,s)}{p(x,s)} dx = p(s) \mathbb{E}_{p(x|s)} r(x,s) dx = p(s)$$

As a result, considering the uncertainty set $\tilde{\cal R}=\{r \in {\cal R}| \forall s, \mathbb{E}_{
p(x|s)} \ r(x, s)= 1\} $ actually implies that the induced prior $q(s)$ equals the sensitive prior $p(s)$ observed in the dataset.

\subsubsection{Analysis of Normalized Weights for a Fully Fair Model} 
\label{theoanalysis}


In the context of fairness, we argue that ensuring a classical global normalization constraint $\mathbb{E}_{p(x,s)}  r(x, s) = 1$ during optimization is not sufficient to provide an accurate behavior with regards to the fair metric, as it may lead to concentrate the majority of the fairness effort to a specific sensitive subgroup. 

To understand why, let us go back to the justification of using an adversarial for group fairness as introduced in \cite{zhang2018}, taking demographic parity as an illustrative example (while the same remains true for other objectives, such as equalized odds). The objective of Demographic Parity is to obtain a classifier that allocates equal chances to both subgroups of the population given the sensitive $S$. Formally, a classifier $f$ that respects: $P(\hat{Y}=1|S=1)=P(\hat{Y}=1|S=0)$, where we note $\hat{Y}=\hat{f}(X)=I(f(X)>0.5)$. Equivalently in the binary case, we need to ensure $\mathbb{E}[\hat{Y}|S=1]=\mathbb{E}[\hat{Y}|S=0]$, or again $P(\hat{Y}=1|S=1)=P(\hat{Y})$. During optimization, it is difficult to build regularization terms on comparison of such quantities, both of them implying the classifier. Besides the non-differentiability induced by the use of an indicator function in the loss, it would require to backpropagate through estimators from 
both populations, which may either be subject to high variance or intractability. 
Rather, adversarial fairness \cite{zhang2018} proposes to use an adversary network $g$ that attempts to reconstruct the sensitive $S$ from $f(X)$. The adversary is optimal when it outputs $\mathbb{E}[S|f(X)]$ accurately for any input, and the classifier is fully fair when $\mathbb{E}[S|f(X)]=P(S=1|f(X))=P(S=1)$.  In that case, we indeed have $P(S=1,f(X))=P(S=1)P(f(X))$, which means independence and finally $P(\hat{f}(X)|S)=P(\hat{f}(X))$. 

Now, let us assume a situation where the classifier is fully fair, which thus means that $P(S|f(X))=P(S)$ for any $x \in X$. Let us also assume that we use an optimal adversary $g^*(f(X))$, that accurately outputs $P(S|f(X))$ for each $(X,S)$ given $f$. Let us look at 
the optimal values of $r$ given $f$ by using our non-parametric formulation BROAD (see section \ref{broad}).  
While using a single global normalization constraint $\mathbb{E}_{p(x,s)}  r(x, s) = 1$ in such a situation of a fully fair classifier,  
we have: $r_i \propto e^{-\mathcal{L}_{S}(g^*(f(x_i)), s_i))/\tau} = e^{\log(P(S=s_i))/\tau}$ for any sample $i$ of the dataset.  Without loss of generality, let us consider that   $P(S=1)>P(S=0)$. In that case, we thus have:  $e^{\log(P(S=1))/\tau}/Z>e^{\log(P(S=0))/\tau} /Z$,  
for any constant $Z>0$, including $Z=\frac{1}{n}\sum_{(x_i,s_i)\in \Gamma} e^{\log(P(S=s_i))/\tau}$.  
This means that in that situation, 
any sample $i$ such that  $s_i=1$ obtains a weight $r_i$ that is greater than the one for any sample $j$ such that $s_j=0$. In other words, with a global normalization constraint, samples from the most populated demographic-sensitive group are more constrained than other ones. Improving accuracy will thus be easier for the least populated group, leading to a model with unbalanced error between the two populations.  

On the other hand,  when considering the conditional validity constraints, we get for that situation: 

$$r_i=\frac{e^{\log(P(S=s_i))/\tau}}{\frac{1}{n_{s=s_i}}\sum_{j,s_j=s_i} e^{\log(P(S=s_i))/\tau}} =1  \qquad \forall i \in {\Gamma}$$  
which ensures a uniform fairness effort for every sample point from the dataset. In that ideal situation, this means that $q(x,s)=p(x,s)$, coming back to a classical adversarial fairness approach such as in \cite{zhang2018}.

\subsubsection{Empirical Observations under Practical Settings}
\label{empanalysis}

\begin{figure}[h!]
    \centering
    \includegraphics[width=0.3\linewidth]{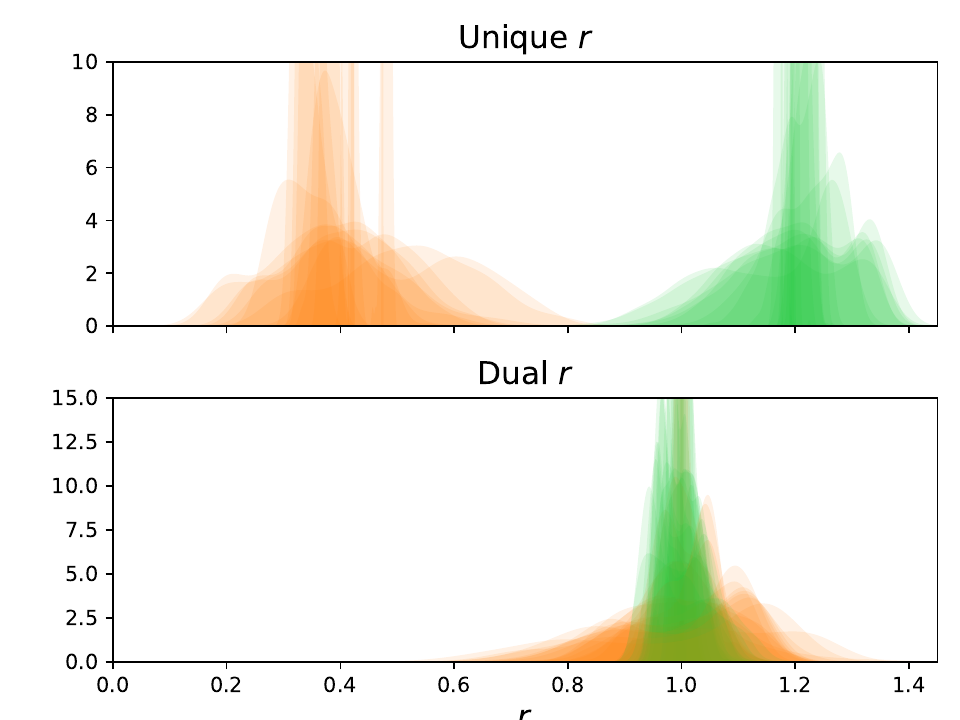}
    \includegraphics[width=0.3\linewidth]{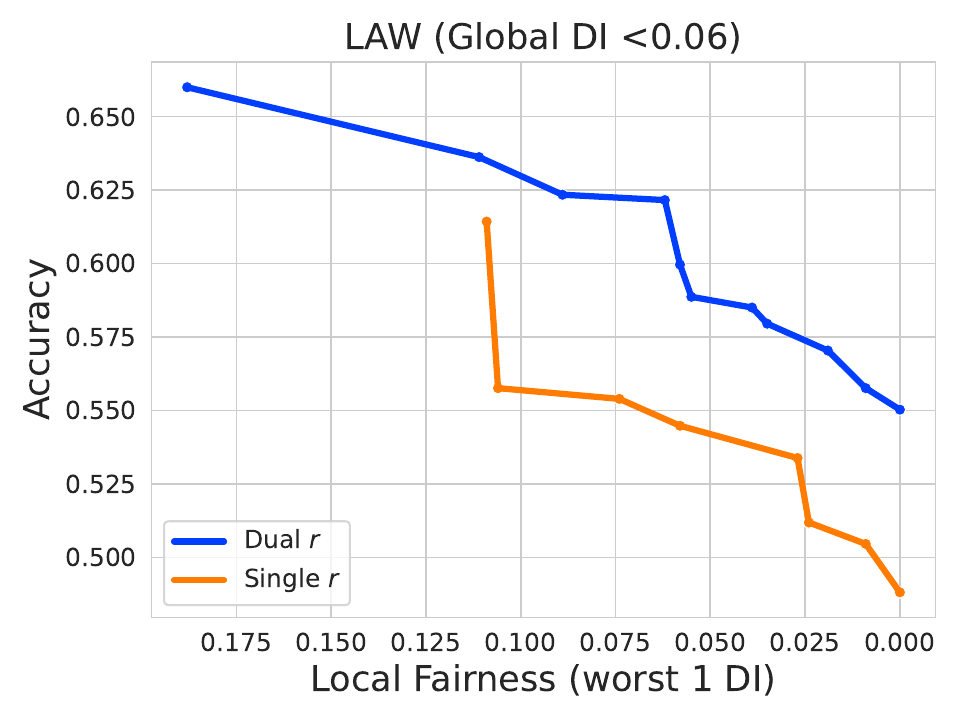}
    \caption{Effect of the conditional constraint ${\mathbb{E}_{(x|s=1)} \ r(x, s) = \mathbb{E}_{(x|s=0)} \ r(x, s) = 1}$. Left: distributions of $r(x|s=0)$ (green) and $r(x|s=1)$ (orange). Right: Pareto fronts of the Accuracy-Local Fairness scores for the Law dataset.
    }
    \label{fig:ablation-dualcont}
\end{figure}

In the previous section, we have shown that ROAD with a single global validity normalization over-constrains the most populated group regarding the sensitive attribute, in the case of a fully fair classifier. We claim that this analysis of an extreme setting highlights a general behavior of the method, which better balances fairness constraints when using our conditional validity normalization.  

In this section, we propose to verify this assumption empirically, by considering the Law dataset, where the sensitive distribution is heavily imbalanced ($P(S=1)\approx(1/4)P(S=0)$). 
In the left image of Fig.~\ref{fig:ablation-dualcont} 
is shown the distribution of $r$ values assigned to instances verifying $S=1$ (orange) and $S=0$ (green) for various training iterations, using a global normalization in the top-most graph, and the conditional normalization on the bottom-most one. Using the global normalization, we observe that no overlap ever happens between the $r$ values of these two sensitive groups.  As a result, most part of the fairness effort is supported by individuals from the largest population $S=0$. 
This is not the case with our proposed conditional normalization, which results in better local fairness (right image of fig.~\ref{fig:ablation-dualcont}).


\subsection{Theoretical Proof: BROAD is a Boltzmann distribution}
\label{sec:app-proof-BROAD}

This section contains proof for Lemma~\ref{BROAD_lemma}, that we rewrite below:

\begin{lemma} 
(Optimal Non-parametric Ratio) 
Consider the optimization problem (inner maximization problem, on $w_r$, of Equation~\ref{eq:final-pb}):
\begin{align}
\max_{w_r} \mathbb{E}_{(x, y, s) \sim p}{({r}(x,s)(\mathcal{L}_{S}(g_{w_g^*}(f_{w_f}(x)), s))} + \tau \, \mathbb{E}_{(x, y, s) \sim p}{({r}(x,s)\log{r}(x,s)}  \\
\text{with } w_g^* = \argmin_{w_g} \mathbb{E}_{(x, y, s) \sim p} \mathcal{L}_{S}(g_{w_g}(f_{w_f}(x)), s)
 \nonumber \\
 \text{Under the global validity constraint: } \mathbb{E}_{(x,y,s)\sim p}\,r(x, s) = 1 \nonumber 
\end{align}
The solution to this problem can be rewritten as an inverted Boltzmann distribution with a multiplying factor $n$:
$$ {r}(x_i,s_i) = \frac{\e^{-\mathcal{L}_{S}(g_{w_g}(f_{w_f}(x_i)), s_i)))/\tau}}{\frac{1}{n}
\sum_{j=1}^{n}\e^{-\mathcal{L}_{S}(g_{w_g}(f_{w_f}(x_j)), s_j))/\tau}}
$$
\end{lemma}

\begin{proof}
The proof results from a direct application of the Karush-Kuhn-Tucker conditions.
Since there is no ambiguity, we use the following lighter notation for the sake of simplicity: $r$ instead of $r(x,s)$; $\mathcal{L}_S$ instead of $\mathcal{L}_{S}(g_{w_g^*}(f_{w_f}(x)), s))$; and $\mathbb{E}$ instead of $\mathbb{E}_{(x, y, s) \sim p}$.
Using these notations and writing the above problem in its relaxed form w.r.t. the distribution constraint, we obtain the following formulation:
$$
\max_{r} \mathbb{E}{(r\mathcal{L}_{S})} +\tau \, \mathbb{E}{({r} \log{r})}
- \kappa ( 1 - \mathbb{E}(r)) 
 \nonumber
$$

Following the Karush-Kuhn-Tucker conditions applied to the derivative of the Lagrangian function~$L$ of this problem in $r_i$, we obtain:

\begin{equation}    
\frac{\partial L}{\partial r_i}=0 \Leftrightarrow \mathcal{L}^i_{S} + \tau (1+ \log r_i)  + \kappa = 0 \Leftrightarrow r_i= \e^{\frac{-\kappa - \mathcal{L}^i_S}{\tau} -1}
\label{eq:proof-1}
\end{equation}

With $\mathcal{L}^i_S=\mathcal{L}_S(g_{w_g^*}(f_{w_f}(x_i)), s_i)$ the i-th component of $\mathcal{L}_S$.
The KKT condition on the derivative in $\kappa$ gives:
${\frac{\partial L}{\partial \kappa}=0 \Leftrightarrow \mathbb{E}(r) = 1}$. 
Combining these two results, we thus obtain: 
$$
\mathbb{E}(r) = \frac{1}{n} \sum_{j = 1}^n r_j = \frac{1}{n} \sum_{j = 1}^n \e^{\frac{-\kappa - \mathcal{L}^j_S}{\tau} -1} = 1 \Leftrightarrow \e^{-\frac{\kappa}{\tau} - 1} = \frac{1}{\frac{1}{n}\sum_{j=1}^n \e^{-\frac{ \mathcal{L}^j_S}{\tau}}}
$$

Which again gives, reinjecting this result in Eq.~\ref{eq:proof-1}:
$$r_i = \frac{\e^{-\frac{\mathcal{L}^i_S}{\tau}}}{\frac{1}{n} \sum_{j=1}^n \e^{-\frac{\mathcal{L}^j_S}{\tau}}}$$
This leads to the form of a Boltzmann distribution (ignoring a multiplication factor $n$), which proves the result.
\end{proof}

The same proof can be derived under the conditional validity constraint.


\subsection{Additional Ablation Studies}
\label{App:ablation_studies}

In this section, we present several additional results, that were not included in the paper due to lack of space.

\subsubsection{Does the complexity of the two adversarial
networks have an impact on the quality of the results?}

In this section, we investigate the impact of the complexity of the two adversarial networks $g$ and $r$ on the quality of the results. 
Previous works, such as \cite{grari2019fair}, have studied the influence of the adversary $g$ complexity on the context of fair adversarial learning objectives and show that a complex architecture can help to achieve significantly better results than a simple one (i.e., a logistic regression).  

Figure~\ref{fig:ablation-complexities} presents the results of this study on the LAW dataset. We compare three levels of complexity for each network: simple (linear predictor), medium (only one hidden layer with 32 neurons with a ReLu activation function), and complex (three hidden layers: 64 neurons, ReLu, 32 neurons, Relu, 16 neurons, Relu). The results show that a more complex adversary $g$ tends to yield better results, which is consistent with the observations made by \cite{grari2019fair} (i.e., a complex network achieves better performance than a simple logistic regression). One possible reason is that a single logistic regression may not retain enough information to predict the sensitive attribute accurately.

In contrast, for network $r$, a simpler linear prediction such as logistic regression is quite efficient (please note that due to the exponential parametrization, an exponential activation is applied to it), while a more complex architecture results in lower performance. 
As explained in Section~\ref{sec:param_app}, the network architecture plays a crucial role in determining the local behavior of $r_{w_r}$. More complex networks tend to favor more local solutions for a given value of $\tau$. In particular, a network of infinite capacity that completes the training will, in theory, exhibit the same behavior as BROAD, thereby yielding more pessimistic solutions. 

This underscores the importance of carefully considering the complexities of the adversaries when designing a fair adversarial learning framework. Our analysis suggests that a more complex adversary $g$ is preferable, while a simpler network $r$ can be more efficient. Further research is needed to better understand the interplay between these complexities and develop strategies for selecting the optimal combination for a given problem setting.



\begin{figure}[h!]
    \centering
    \includegraphics[width=0.45\linewidth]{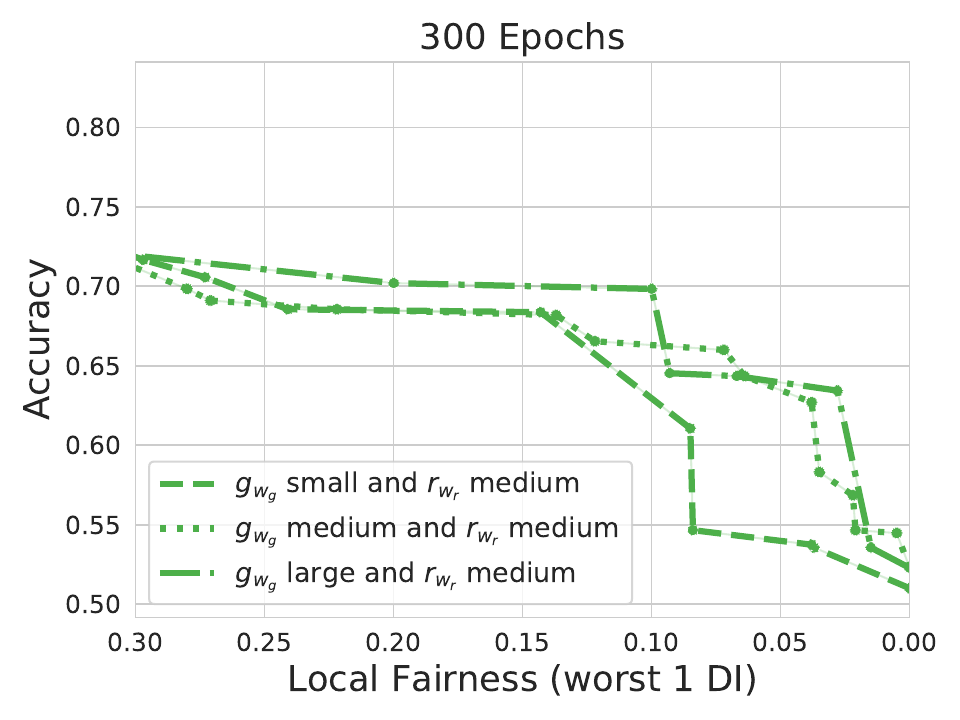}
    \includegraphics[width=0.45\linewidth]{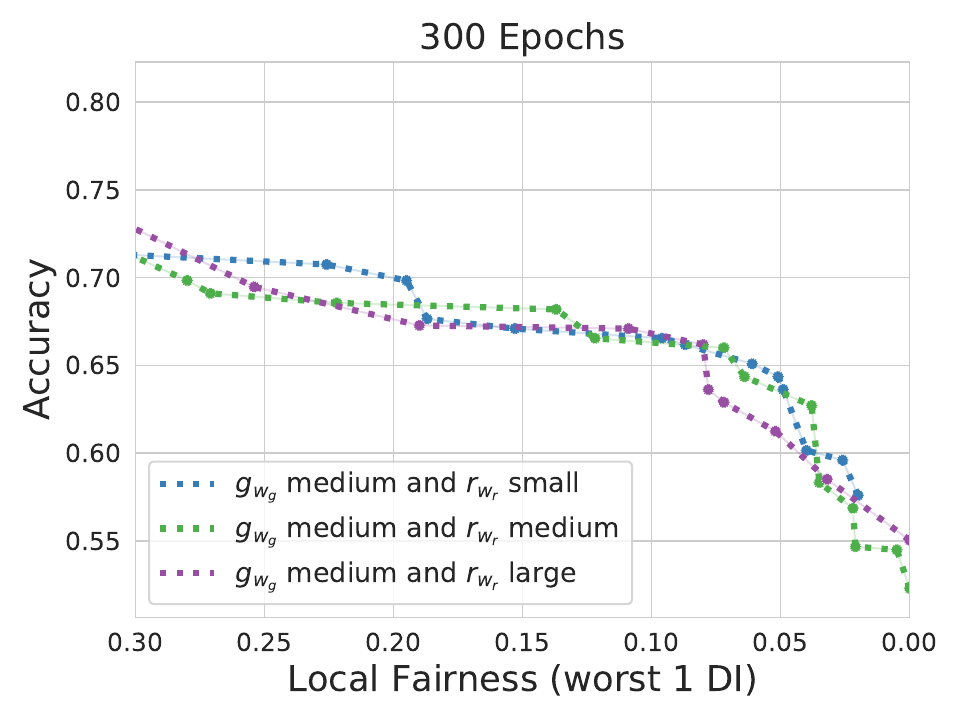}
    \caption{Impact of the complexity of adversarial networks g and r on the quality of results on the  LAW dataset. The comparison is carried on three complexity levels: simple (a simple regression), medium (only one hidden layer), and complex (three hidden layers). The curves represented are, for each method, the Pareto front for the results satisfying the imposed global fairness constraint (here, Global DI $< 0.05$ for all datasets). 
    }
    \label{fig:ablation-complexities}
\end{figure}

\subsubsection{Assessing Local Fairness on the Worst-3-DI instead of the Worst-1-DI
}

\begin{figure}[ht!]
    \centering
    \includegraphics[width=0.32\linewidth]{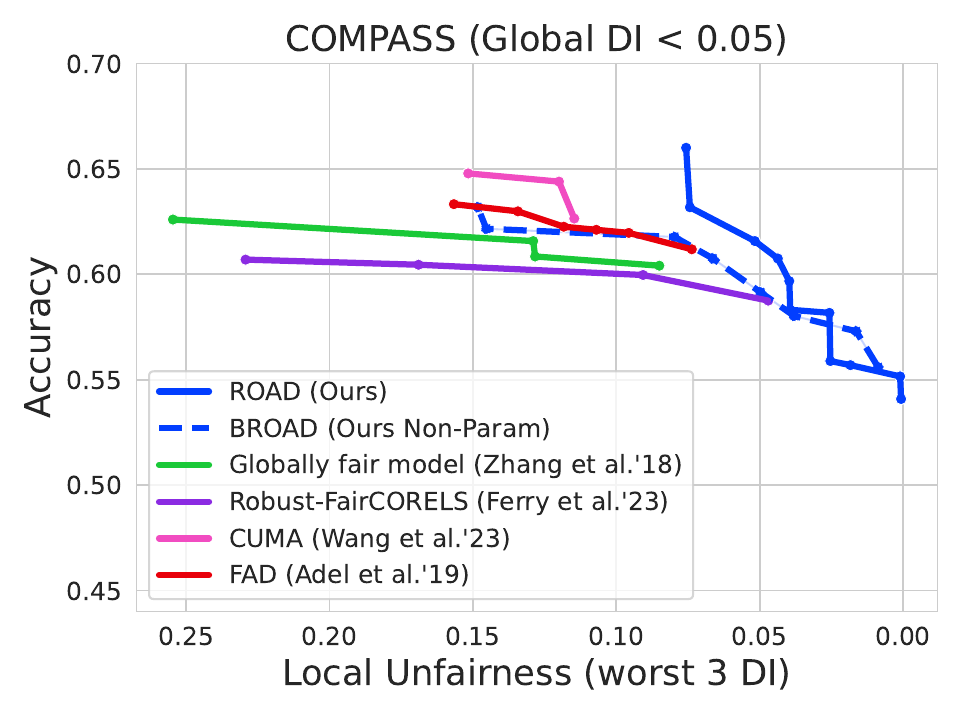}
    \includegraphics[width=0.32\linewidth]{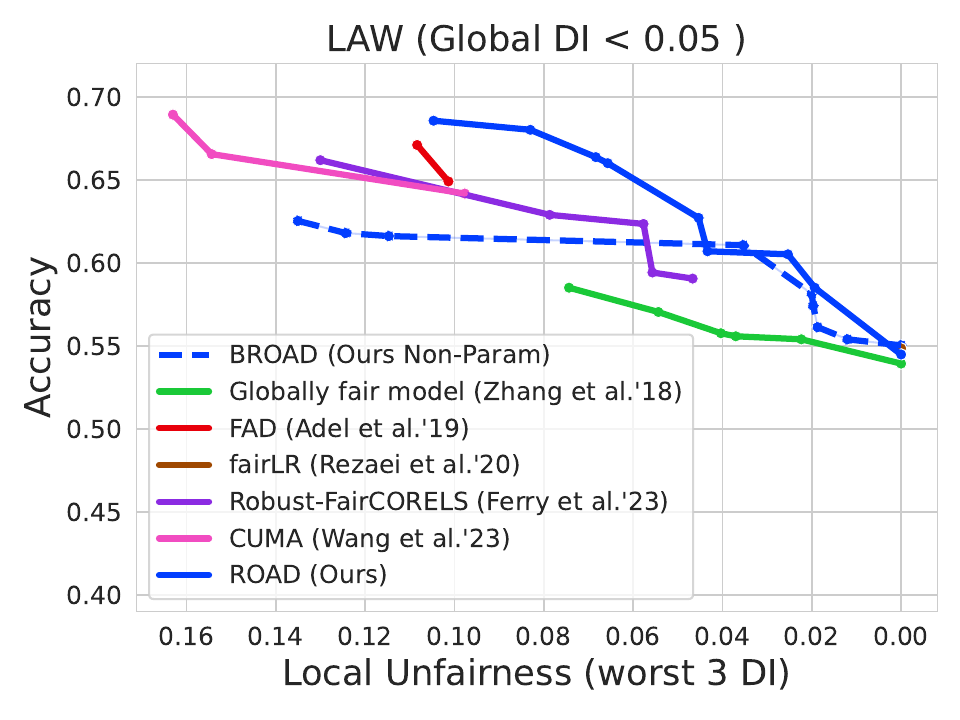}
    \includegraphics[width=0.32\linewidth]{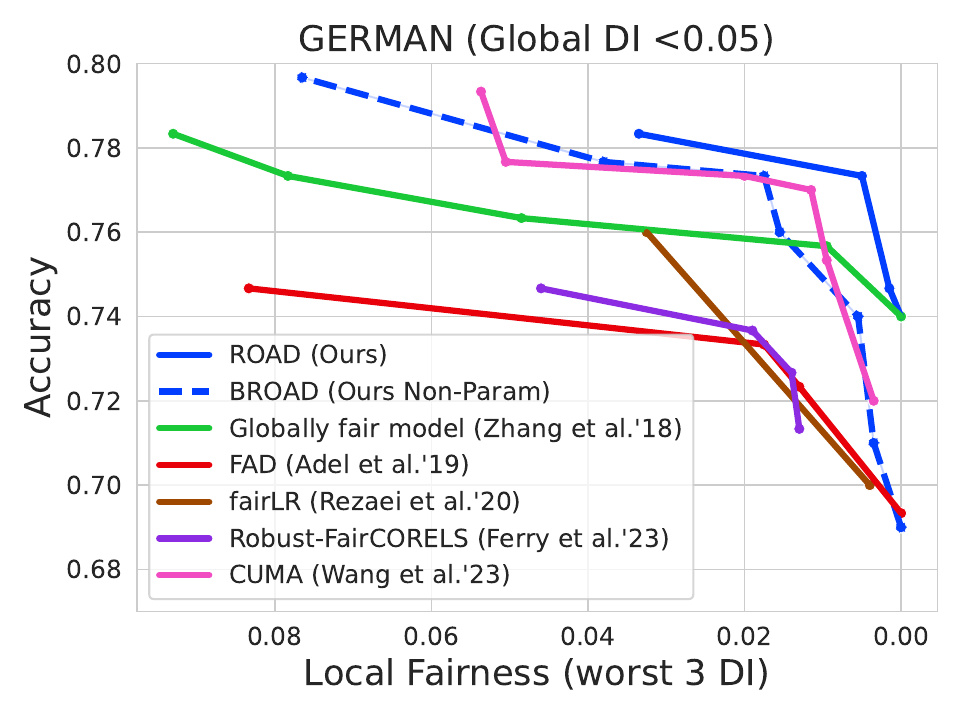}
    \caption{Results for the experiment on Local Fairness. For all datasets, the X-axis is Worst-3-DI, Y-axis is Global accuracy. The curves represented are, for each method, the Pareto front for the results satisfying the imposed global fairness constraint (here, Global DI $< 0.05$ for all datasets). 
    }
    \label{fig:results-localfairness_worst3DI}
\end{figure}

In this experiment, we aim to gain a deeper understanding of our methods' behavior. Instead of assessing local fairness based on the worst disparate impact value for the globally defined subgroups, we focus on the worst 3. This approach allows us to observe the impact of our instance-level re-weighting strategy at an intermediate level.

The definition of our segment remains consistent as in Section~\ref{sec:experiments-localfairness}. %
In Figure~\ref{fig:results-localfairness_worst3DI} we show the resulting Accuracy-Worst-3-DI Pareto curves for each method. We observe that the results are slightly similar to those observed on the Worst-1-DI in the main text.  For all datasets, as in the Worst-1-DI experiment ROAD mostly outperforms all other methods. On the other hand, BROAD sometimes reaches comparable performance, illustrating the overly pessimistic solutions mentioned earlier. 

\subsubsection{Additional results for the Ablation study on subgroup definition (Section~~\ref{sec:experiments-ablation-subgroups})
}

Section~\ref {sec:experiments-ablation-subgroups} described the experimental results 
on the adaptability of ROAD to various subgroup definitions. We complete these results with the following: in Figure~\ref{fig:app-ablation-subgroups-5}, the experiments are repeated five times, and the same conclusions are observed: It is consistently below the values reached by
the globally fair model~\cite{zhang2018} In addition; in Figure~\ref{fig:app-ablation-subgroups-curves}, the details of the same results are presented. Instead of only showing the Worst-1-DI values, all local DI values, for every subgroup, are shown.

\begin{figure}
    \centering
    \includegraphics[width=0.32\linewidth]{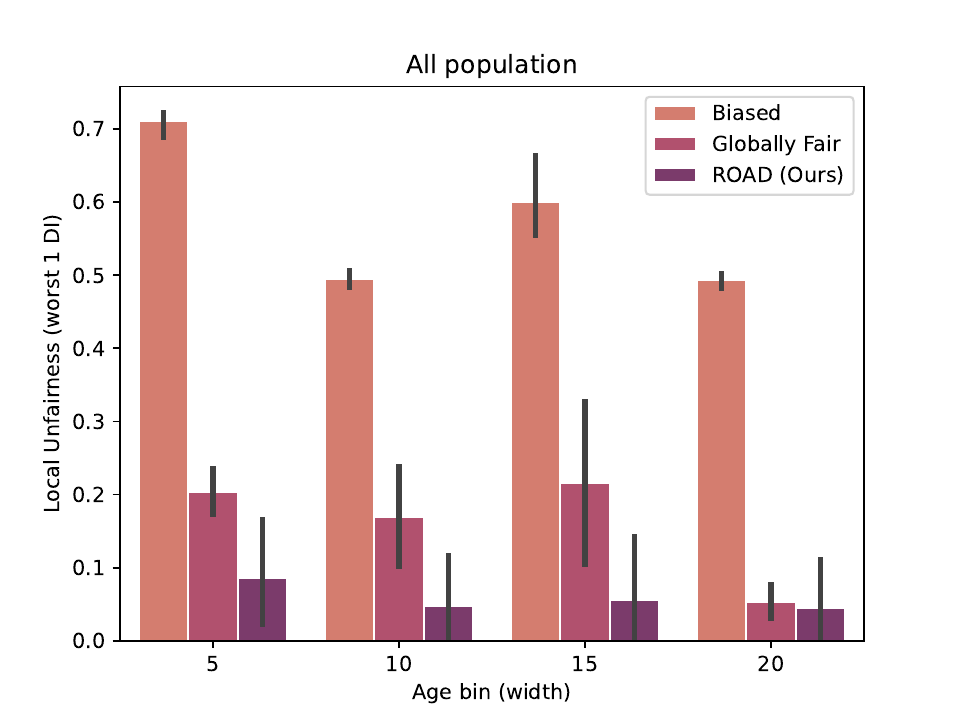}
    \includegraphics[width=0.32\linewidth]{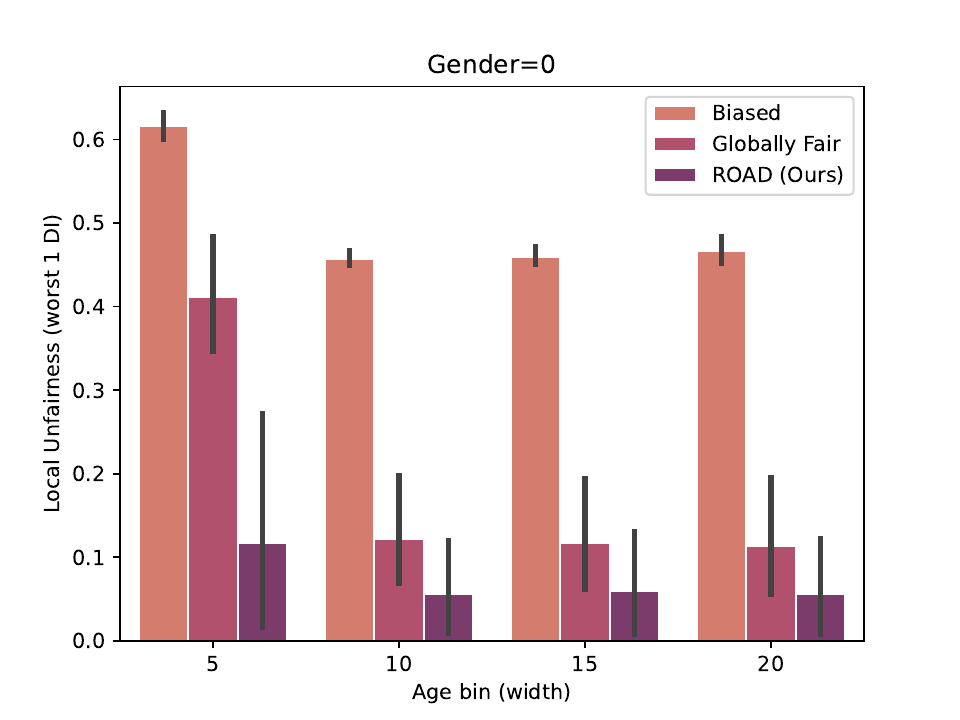}
    \includegraphics[width=0.32\linewidth]{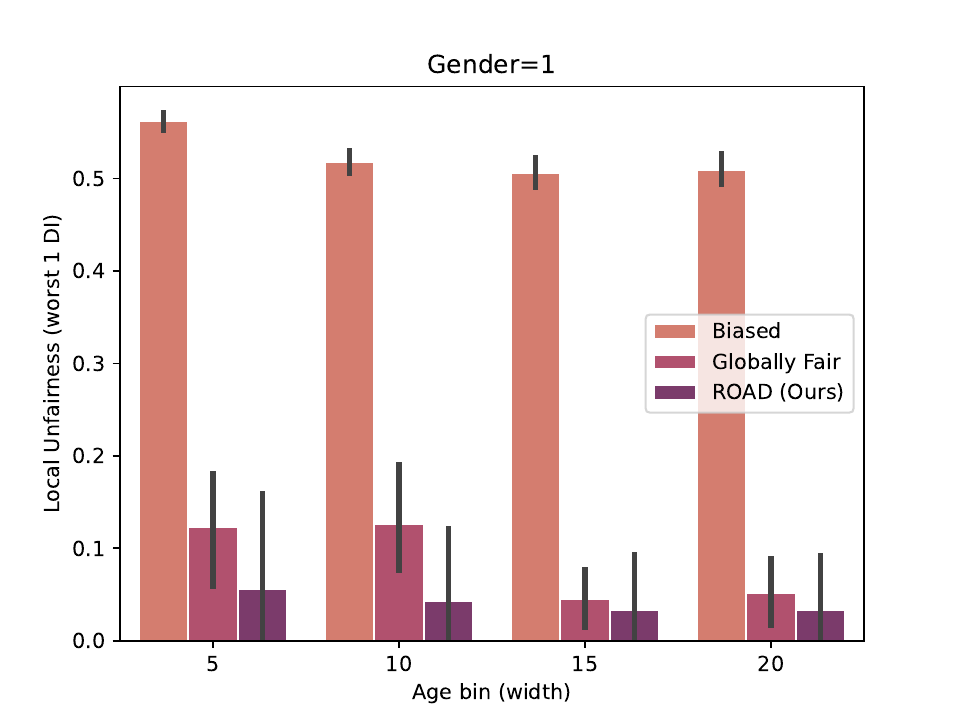}
    \label{fig:app-ablation-subgroups-5}
    \caption{Local DI scores for subgroups of the Law dataset defined as age bins (top row)
and age bins and gender attributes (middle and bottom rows). Each column corresponds to bins of
different sizes. Subgroups with less than 10 individuals are ignored. The process has been iterated five times. 
    }
    \label{fig:localfairness_subpop}
\end{figure}


\begin{figure}
    \centering
    \includegraphics[width=0.23\linewidth]{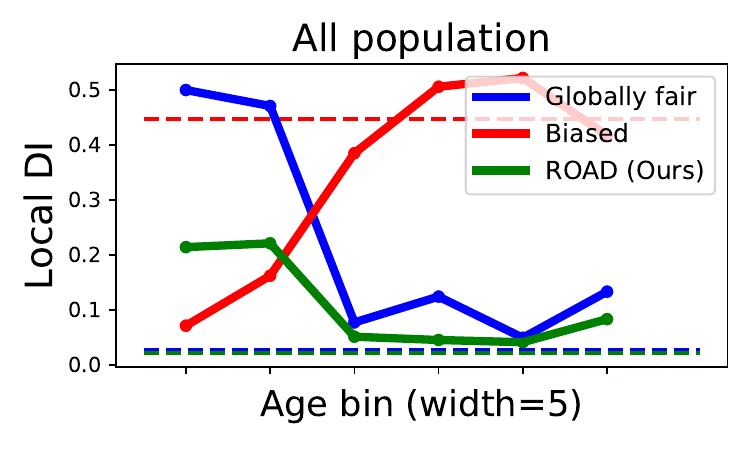}
    \includegraphics[width=0.23\linewidth]{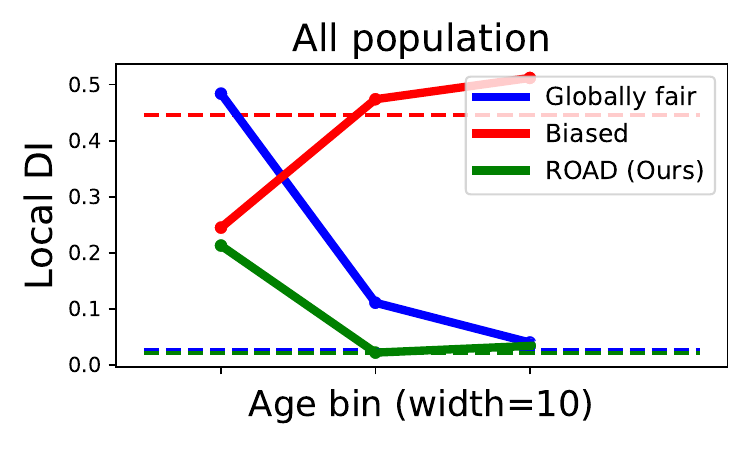}
    \includegraphics[width=0.23\linewidth]{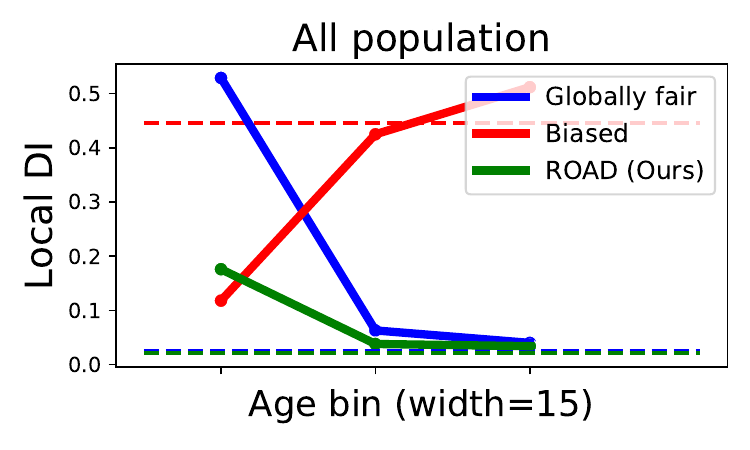}
    \includegraphics[width=0.23\linewidth]{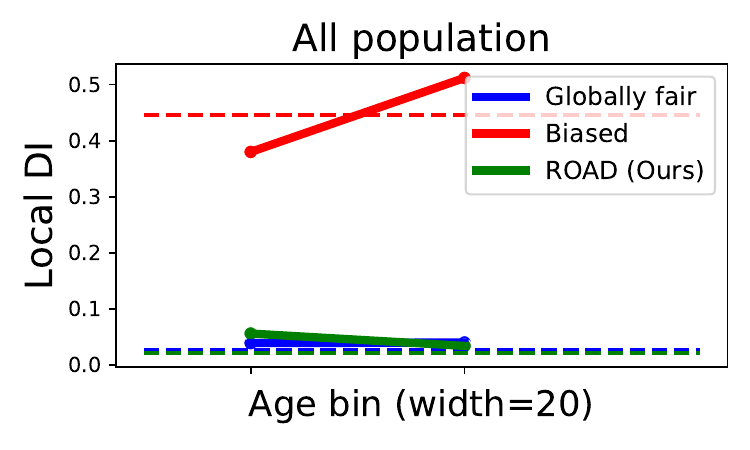}
    \includegraphics[width=0.23\linewidth]{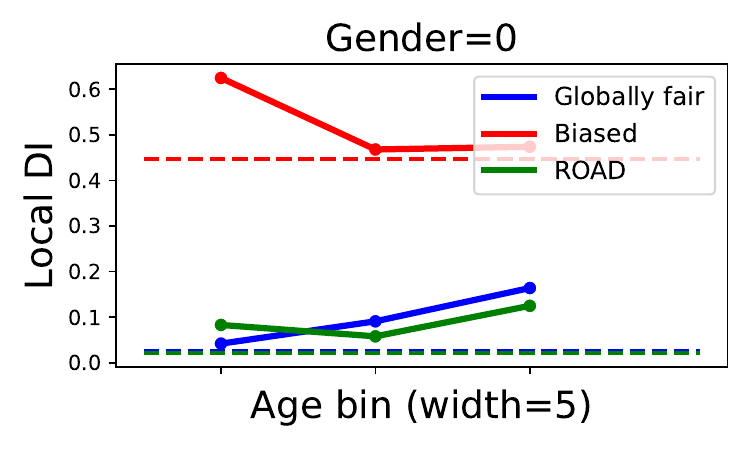}
    \includegraphics[width=0.23\linewidth]{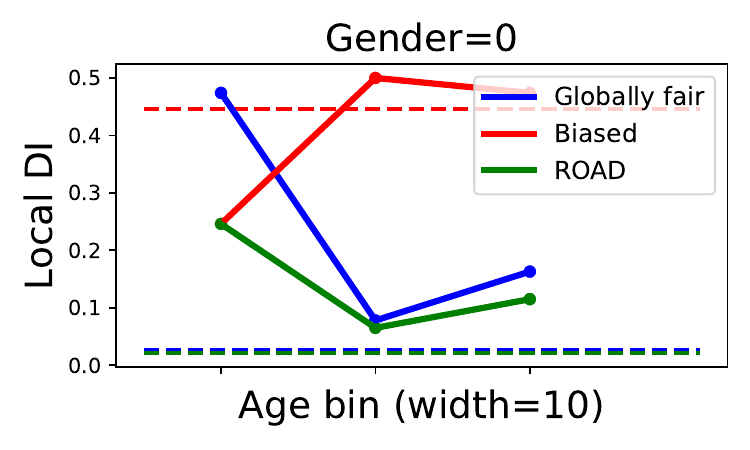}
    \includegraphics[width=0.23\linewidth]{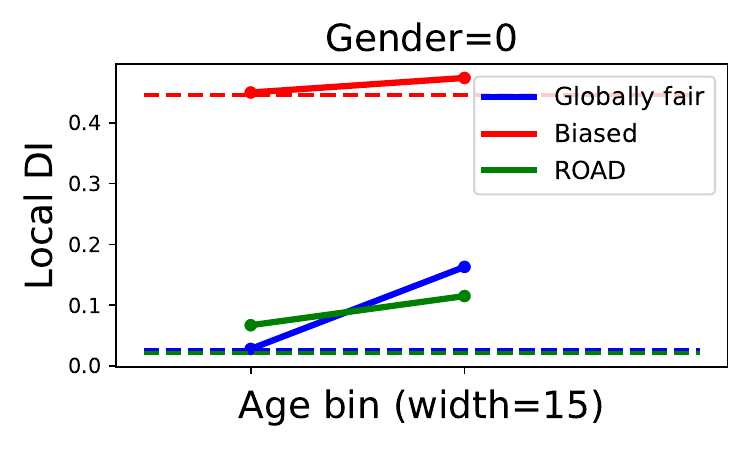}
    \includegraphics[width=0.23\linewidth]{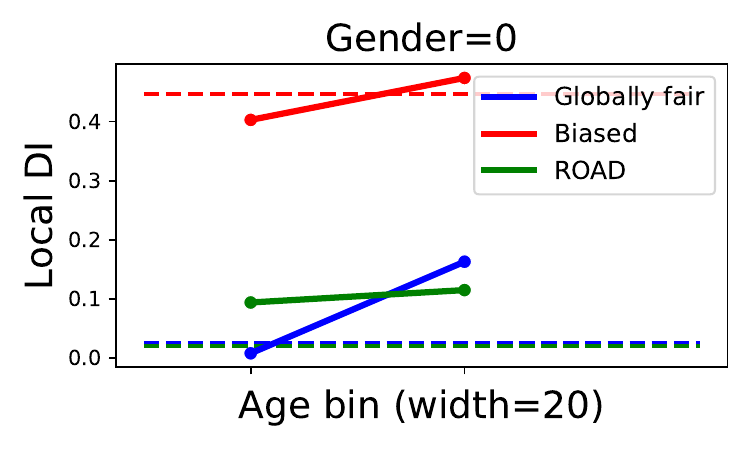}
    \includegraphics[width=0.23\linewidth]{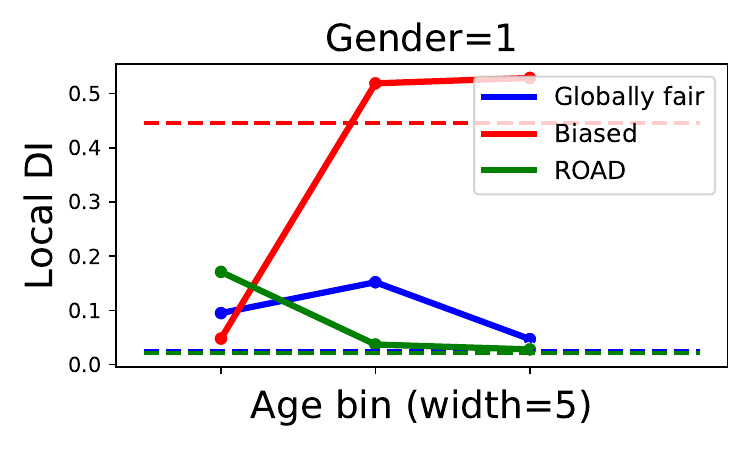}
    \includegraphics[width=0.23\linewidth]{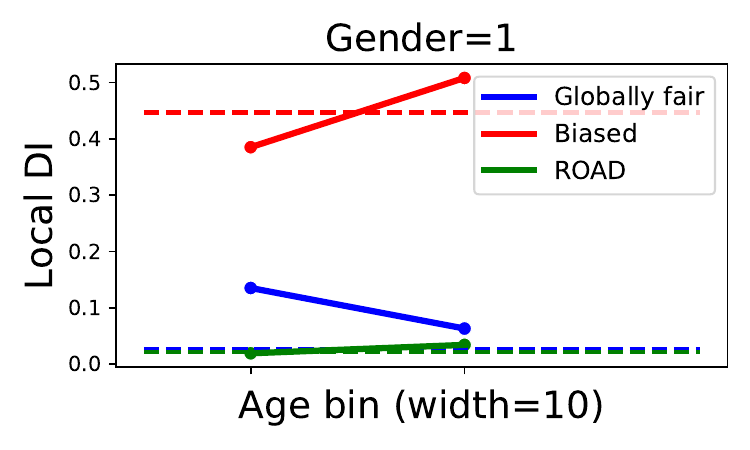}
    \includegraphics[width=0.23\linewidth]{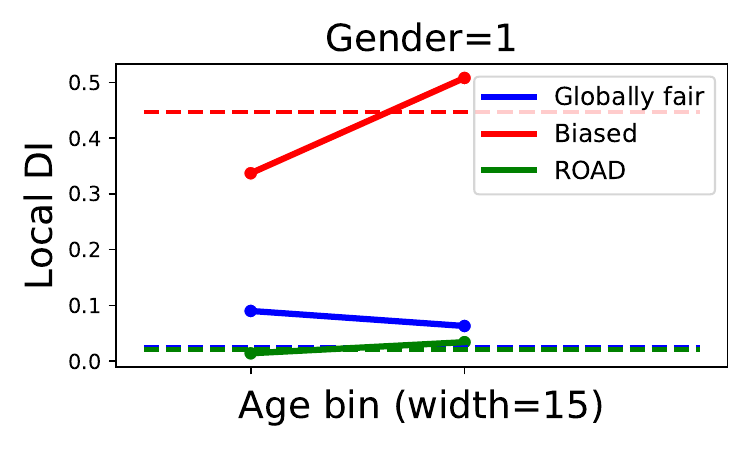}
    \includegraphics[width=0.23\linewidth]{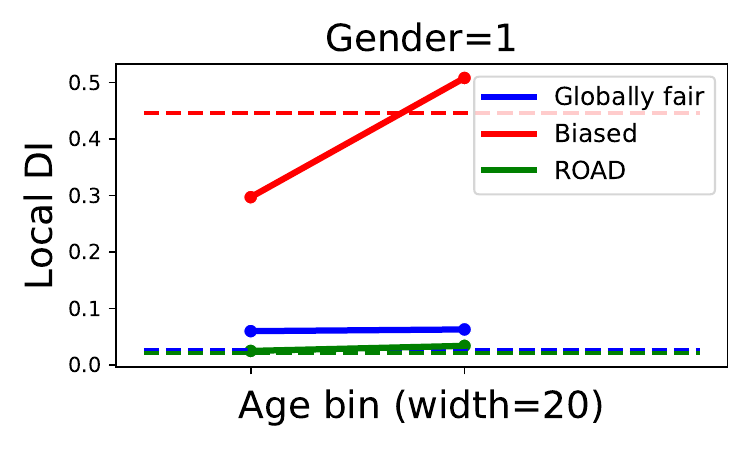}
    \label{fig:app-ablation-subgroups-curves}
    \caption{Local DI scores for subgroups of the Law dataset defined as age bins (top row) and age bins and gender attribute (middle and bottom rows). Each column corresponds to bins of different sizes. Subgroups with less than 10 individuals are ignored. 
    }
    \label{fig:app_ablation-subgroups}
\end{figure}

\subsection{Additional Experimental Details}
We provide additional details on the experimental evaluations.
\label{App:Expe_details}

\subsubsection{Datasets Description}
\label{sec:app-datasets}

\paragraph{Experiment 1: Assessing Local Fairness}
\begin{itemize}
\item \textbf{Compas} The COMPAS data set~\cite {angwin2016machine} contains 13 attributes of about 7,000 convicted criminals with class labels that state whether or not the individual recidivated within $2$ years. Here, we use race as sensitive attribute, encoded as a binary attribute, Caucasian or not-Caucasian.
\item  \textbf{Law} The Law School Admission dataset \cite{wightman1998lsac} contains $10$ features on 1,823 law school applicants, including LSAT scores, undergraduate GPA, and class labels representing their success in passing the bar exam. The sensitive attribute selected in this dataset is race, a binary attribute.

\item \textbf{German} The German UCI credit dataset \cite{german_credit_data} comprises credit-related data for 1,000 individuals, including 20 attributes such as credit history, purpose, and employment status, as well as class labels that represent their credit risk (good or bad). For the sensitive attribute, we consider the gender attribute with Males as the privileged group, and Females as the unprivileged group. 

\end{itemize}

\paragraph{Experiment 2: Distribution Drift}

The training set used for this experiment is the training set of the traditional Adult dataset.

The test sets are the following:
\begin{itemize}
    \item in-distribution: traditional test set of the Adult dataset
    \item 2014: Adult 2014 dataset from Folktables, see below
    \item 2015: Adult 2015 dataset from Folktables, see below
\end{itemize}

Adult dataset contains US Census
data collected before 1996. 
Its very frequent use in the fairness literature was discussed in~\citet{ding2021retiring}, in which the authors criticize its out-of-date aspect, especially in the context of evaluating social aspects of algorithms such as their fairness. Instead, they introduce the Folktables, up-to-date, aligned, versions of the Adult dataset. Using these datasets, the authors discussed the generalization of the fairness metrics over time. A similar setup was then considered by~\citet{wang2023robust} to evaluate how robust their method, CUMA, was to temporal drift. In our work, we directly replicate their experimental setup.

\subsubsection{Methods Description}
\label{sec:app-competitors}
Besides ROAD and BROAD, proposed in this work, the following method are used as competitors and baselines for the experiments:

\begin{itemize}
    \item Globally fair model~\citep{zhang2018}: described in Section~\ref{sec:background-groupfairness}. To obtain the Pareto curves in the results, the~$\lambda$ hyperparameter is set to various values, allowing the exploration of the Fairness-accuracy tradeoff. As described in the footnote in Section~\ref{sec:proposition-road}, this approach can be adapted to optimize either Demographic Parity or Equalized Odds.
    \item FAD~\citep{adel2019one}: similarily to~\cite{zhang2018}, FAD relies on an adversarial network to reconstruct the sensitive attribute and thus estimate bias. However, instead of reconstructing the sensitive attribute from the predictions~$f(x)$, FAD reconstructs it from an intermediary latent space. The authors have shown that this helped obtaining better results. Similarly to~\cite{zhang2018}, FAD relies on a single hyperparameter~$\lambda$ to balance between accuracy and fairness. Originally optimizing Demographic Parity, the same adaption as the approach from~\citet{zhang2018} can be made to optimize EO.
    \item Robust FairCORELS~\citep{ferry2022improving}: this method leverages the DRO framework for fairness and integrates in the FairCORELS approach~\citep{aivodji2021faircorels}. It relies on preprocessing step discretizing the whole dataset into segments. The uncertainty set~$\mathcal{Q}$ for DRO is then defined as a Jaccard Index-ball around $p$. RobustFairCORELS then uses linear programming and efficient pruning strategies to explore~$\mathcal{Q}$ and solve its otpimization problem. 
    RobustFairCORELS relies mainly on one hyperparameter:
    $\epsilon$, the unfairness tolerance allowed, from which is derived the \emph{size} of $\mathcal{Q}$ (maximum Jaccard distance from $p$ allowed).
    Besides, it relies on FairCORELS hyperparameters, which we set to their default values. RobustFairCORELS can be used to optimize several fairness criteria, including Demographic Parity and Equalized Odds.
    \item fairLR~\citep{rezaei2020fairness}: to achieve more robust fairness, fairLR also leverages the DRO framework for fairness. 
    The generalization they propose to ensure is more restrictive, as they make the assumption that there is no drift between the train and test distributions. They therefore define the ambiguity set~$\mathcal{Q}$ as the set of distributions matching some statistics of the training data, that they use to train logistic regressions. These constraints can be adapted either to optimize for Demographic Parity or Equalized Odds.
    fairLR relies only one hyperparameter, $C$, a regularization parameter for the logistic regression. Although it has seemingly no direct link with fairness, we still consider it in the experiments. 
    \item CUMA~\citep{wang2023robust}: this method leverages adversarial learning to mitigate fairness. Besides, it proposes to add a new component to its loss function, based on the curvature smoothness of its loss function.
    The idea is that to ensure better fairness generalization, the curvatures of the loss functions of each sensitive group should be similar. The objective function is then approximated using a neural network. It uses two hyperparameters, to balance accuracy with fairness and smoothness: $\alpha$ (adversarial component) and $\gamma$ (smoothness component). Similarly to the Globally Fair model and FAD, CUMA can be adapted in the same manner to either optimize Demographic Parity of Equalized Odds.  
\end{itemize}

\subsubsection{Subgroups Description}
\label{sec:app-subgroups-description}

In this section, we describe how subgroups were defined in the experiments. 

\paragraph{Experiment 1}

\begin{itemize}
    \item \textbf{Compas} Subpopulations are created by discretizing the feature \emph{Age} in buckets with a width of $10$, intersected with feature \emph{Gender}. Subgroups of size of at least $50$ are kept.  
    \item \textbf{Law 
    } Subpopulations are created in the same manner, by discretizing the feature \emph{Age} in buckets with a width of $10$, intersected with feature \emph{Gender}. Subgroups of size of at least $20$ are kept. 
    \item \textbf{German 
    }
    Subpopulations are created by discretizing the feature \emph{Age} in buckets with a width of $10$, intersected with feature \emph{Gender}. Subgroups of size of at least $20$ are kept. 
\end{itemize}

\paragraph{Ablation study: how important is the definition of subgroups}
\label{App:defsubgroups}

Table~\ref{tab:app-subgroups} describes how subgroups were defined in the 12 scenarios shown in Figure~\ref{fig:app_ablation-subgroups}.

\begin{table}
    \centering
    \begin{tabular}{c|c|c}
        Subgroup definition & Age bin width & Gender \\
        \hline
        Subgroup def. 0 & 5 & All population  \\
        Subgroup def. 1 & 10 & All population  \\
        Subgroup def. 2 & 15 & All population  \\
        Subgroup def. 3 & 20 & All population  \\
        Subgroup def. 4 & 5 & Gender=0  \\
        Subgroup def. 5 & 10 & Gender=0  \\
        Subgroup def. 6 & 15 & Gender=0  \\
        Subgroup def. 7 & 20 & Gender=0  \\
        Subgroup def. 8 & 5 & Gender=1  \\
        Subgroup def. 9 & 10 & Gender=1  \\
        Subgroup def. 10 & 15 & Gender=1  \\
        Subgroup def. 11 & 20 & Gender=1  \\

    \end{tabular}
    \caption{Description of how subgroups were defined in the 12 scenarios shown in Figure~\ref{fig:ablation-subgroups}.}
    \label{tab:app-subgroups}
\end{table}

\subsubsection{Implementation details for ROAD and BROAD}
\label{sec:app-hyperparameters-architecture}

In order to obtain the results shown in Figure~\ref{fig:results-localfairness} ad~\ref{fig:adult-drift}, we explore the following hyperparameter values for ROAD and BROAD:

\begin{itemize}
    \item $\lambda_g$: grid of 20 values between 0 and 5
    \item $\tau$: grid of 10 values between 0 and 1
\end{itemize}

The networks $f_{w_f}$, $g_{w_g}$ and $r_{w_r}$ have the following architecture:

\begin{itemize}
    \item $f_{w_f}$: FC:64 R, FC:32 R, FC:1 Sig
    \item $g_{w_g}$: FC:64 R, FC:32 R, FC:16 R, FC:1 Sigm
    \item $h_{w_r}$: FC:64 R, FC:32 R, FC:1
\end{itemize}

\subsection{Description of the algorithms}
\label{sec:desc_algo}


\subsubsection{ROAD Algorithm}

The ROAD algorithm, for the Demographic Parity objective, aims to train a locally fair model by iteratively updating three components: the sensitive adversarial model $g_{w_g}$, the adversarial ratio model $r_{w_r}$, and the predictor model $f_{w_f}$.  The following problem of ROAD is defined as follows (cf. Eq.~\ref{eq:fair-dro0}):
\begin{eqnarray}\label{eq:ROAD_NN}
\min_{w_f} \max_{\substack{_{w_r}}}
\frac{1}{n} \sum_{i=1}^{n} \mathcal{L}_{Y}(f_{w_f}(x_i), y_i) - \lambda_g [ \frac{1}{n} \sum_{i=1}^{n}{({\tilde{r}_{w_r}}(x_i,s_i)\mathcal{L}_{S}(g_{w_g^*}(f_{w_f}(x_i)), s_i)))}  \nonumber \\ + \tau \underbrace{\frac{1}{n} \sum_{i=1}^{n}{({\tilde{r}_{w_r}}(x_i,s_i)\log({\tilde{r}_{w_r}}(x_i,s_i))}}_{\text{KL constraint}}]  \\
\text{with } w_g^* = \argmin_{w_g} \frac{1}{n} \sum_{i=1}^{n} \mathcal{L}_{S}(g_{w_g}(f_{w_f}(x_i)), s_i) \nonumber 
\end{eqnarray}
with $\tilde{r}$ representing the normalized ratio per demographic sub-groups.

To obtain better results, as commonly done in the literature, 
we perform several training iterations of the adversarial networks 
for each prediction iteration. This approach results in more robust adversarial algorithms, preventing the predictor classifier from dominating the adversaries.

The required inputs for ROAD include the training data containing $X$, $Y$, and $S$ (please note that $S$ is not required at the testing time), the loss functions (i.e., logloss function applied in our experiment), batch size, number of epochs, neural networks architectures, learning rates, and the control parameters ($\tau$ and $\lambda$).

For each epoch, the algorithm~\ref{alg:ROAD} iterates through all the batches in the training data. In each batch, the algorithm performs the following steps:

\begin{itemize}
    \item Update the sensitive adversarial model $g_{w_g}$: The sensitive adversarial takes the current output predictions $f_{w_f}(x)$ as input. The sensitive adversarial loss is calculated based on the current parameters of the predictor and sensitive adversarial models. The sensitive adversarial model is then updated using traditional gradient descent.
    \item Update the ratio model $r_{w_r}$: The likelihood ratio is normalized across demographic subgroups. The likelihood ratio cost function is calculated with respect to the average of the product between the weighting ratio and sensitive adversarial loss in addition to the KL constraint. Then, the adversarial model is updated using gradient descent.
    \item Update the predictor model: The loss function for the predictor model $f_{w_f}$ is calculated, taking into account the traditional loss function of the predictor and the average of the product between the weighting ratio and sensitive adversarial. The predictor model is then updated using gradient descent.
\end{itemize}

The process is repeated for the specified number of epochs. By iteratively updating these three models, the algorithm aims to achieve a balance between minimizing the loss function and maintaining fairness across demographic subgroups in the resulting predictor model.

\begin{algorithm}
\caption{ 
ROAD: Robust Optimization for Adversarial Debiasing}
\label{alg:ROAD}
\begin{algorithmic}[1] 
\REQUIRE Training set $\Gamma$,\\
 Individual loss functions $l_Y$ and $l_S$, \\
 Batchsize $b$, Number of epochs $n_e$, Neural Networks $r_{\omega_{r}}$, $f_{\omega_f}$ and $g_{\omega_g}$, \\
 Learning rates $\alpha_f$, $\alpha_g$ and $\alpha_r$, Number of training iterations $n_g$ and $n_r$,\\
 Fairness control $\lambda_g \in \mathbb{R}$ and Temperature control $\tau \in \mathbb{R}$


\FOR{ epoch $ \in [1,... n_e]$}

\FORALL{batch $\{(x_{1}, s_1, y_{1}), . . . ,(x_{B}, s_B, y_{B})\}$ drawn from $\Gamma$}

\FOR{$i \in [1,... n_g]$}
\STATE $J_s(\omega_{f},\omega_{g}) \gets \frac{1}{b}\sum_{i=1}^{b}
l_S(g_{w_g}(f_{w_f}(x_i)),s_i)$ \COMMENT{Calculate the sensitive adversarial loss}
\STATE $\omega_{g} \gets \omega_{g} -\alpha_{g} \frac{\partial  J_s(\omega_{f},\omega_{g})}{\partial {\omega_{g}}}$ \COMMENT{Update the sensitive adversarial model $g_{\omega_{g}}$ by gradient descent}
\ENDFOR
\FOR{$j \in [1, ... n_r]$}
\STATE $\forall s \in \{0,1\} \quad r_{w_r}(x_i,s_i =s) \gets \frac{\sum_{i=1}^{b} e^{h_{w_r}(x_i,s_i)}}{\sum_{i=1}^{b} e^{h_{w_r}(x_i,s_i)}\mathds{1}_{s_i =s}}$ \COMMENT{Normalize the likelihood ratio per demographic subgroups}
\STATE $J_r(\omega_{r},\omega_{f},\omega_{g}) \gets \frac{1}{b}\sum_{i=1}^{b}r_{w_r}(x_i,s_i)*l_S(g_{w_g}(f_{w_f}(x_i)),s_i)+\tau * r_{w_r}(x_i,s_i) \log(r_{w_r}(x_i,s_i))$ \COMMENT{Calculate the likelihood ratio cost function}
\STATE $\omega_{r} \gets \omega_{r} -\alpha_{r} \frac{\partial J_r(\omega_{r},\omega_{f},\omega_{g})}{\partial {\omega_{r}}}$ \COMMENT{Update the adversarial model $r_{\omega_{r}}$ by gradient descent}

\ENDFOR
\STATE $\forall s \in \{0,1\} \quad r_{w_r}(x_i,s_i =s) \gets \frac{\sum_{i=1}^{b} e^{h_{w_r}(x_i,s_i)}}{\sum_{i=1}^{b} e^{h_{w_r}(x_i,s_i)}\mathds{1}_{s_i =s}}$ \COMMENT{Normalize the likelihood ratio per demographic subgroups}
\STATE $J_f(\omega_{r},\omega_{f},\omega_{g}) \gets \frac{1}{b}\sum_{i=1}^{b} l_Y(f_{w_f}(x_i),y_i) - \lambda r_{w_r}(x_i,s_i)*l_S(g_{w_g}(f_{w_f}(x_i)),s_i)$ \COMMENT{Calculate the loss function of the predictor model}

\STATE $\omega_{f} \gets \omega_{f} - \alpha_{f}\frac{\partial J_f(\omega_{r}, \omega_{f},\omega_{g})}{\partial \omega_{f}}$ \COMMENT{Update the predictor model}

\ENDFOR
\ENDFOR
\end{algorithmic}
\end{algorithm}

\newpage
\subsubsection{BROAD Algorithm}

The BROAD algorithm describes the non-parametric, implementation for the analytical solution of the problem considered in Eq~\ref{eq:final-pb}. 
The main difference between the two algorithms is that BROAD does not include the update step for the ratio model $r_{w_r}$ and instead calculates the likelihood ratio directly in the predictor model update step. While it simplifies the optimization problem, it has been observed that the uncertainty sets generated by non-parametric approaches are pessimistic (i.e., may encompass the considered distributions), and as a result, lead to sub-optimal solutions.

This non-parametric algorithm, for the Demographic Parity objective, aims to train a locally fair model by iteratively updating two components: the sensitive adversarial model $g_{w_g}$ and the predictor model $f_{w_f}$. The required inputs for this algorithm are the same as those for the previous one.

\begin{itemize}
    \item Update the sensitive adversarial model $g_{w_g}$: This step is the same as in ROAD. The sensitive adversarial loss is calculated based on the current parameters of the predictor and sensitive adversarial models. The sensitive adversarial model is then updated using traditional gradient descent.
    \item Calculate the likelihood ratio: Instead of updating the ratio model $r_{w_r}$ as the ROAD algorithm, the likelihood ratio is calculated directly using the exponential of the negative sensitive adversarial loss divided by the temperature parameter $\tau$. The likelihood ratio is normalized across sensitive groups.
    \item Update the predictor model: The loss function for the predictor model $f_{w_f}$ is calculated, taking into account the traditional loss function of the predictor and the average of the product between the weighting ratio and sensitive adversarial. The predictor model is then updated using gradient descent.
\end{itemize}

The process is repeated for the specified number of epochs. By iteratively updating these two models, the algorithm aims to achieve a balance between minimizing the loss function and maintaining fairness across demographic subgroups in the resulting predictor model.

In summary, the main difference between the two algorithms is the approach to updating the ratio model. ROAD updates the ratio model $r_{w_r}$ through an iterative process, while BROAD calculates the likelihood ratio directly in the predictor model update step. This alternative algorithm simplifies the process and may result in faster computation times. However, when comparing the performance with respect to fairness trade-offs, it has been observed to underperform compared to the original algorithm (ROAD).

\begin{algorithm}
\caption{ 
BROAD: Boltzmann Robust Optimization for Adversarial Debiasing \\ 
\quad \quad  (Conditional normalization)}
\label{alg:BROAD}
\begin{algorithmic}[1] 
\REQUIRE Training set $\Gamma$,\\
 Individual loss functions $l_Y$ and $l_S$, \\
 Batchsize $b$, Number of epochs $n_e$, Neural Networks $f_{\omega_f}$ and $g_{\omega_g}$, \\Learning rates $\alpha_f$, $\alpha_g$ and $\alpha_r$, Number of training iterations $n_g$, \\ 
 Fairness control $\lambda_g \in \mathbb{R}$ and Temperature control $\tau \in \mathbb{R}$


\FOR{ epoch $ \in [1,... n_e]$}

\FORALL{batch $\{(x_{1}, s_1, y_{1}), . . . ,(x_{B}, s_B, y_{B})\}$ drawn from $\Gamma$}

\FOR{$i \in [1,... n_g]$}
\STATE $J_s(\omega_{f},\omega_{g}) \gets \frac{1}{b}\sum_{i=1}^{b}
l_S(g_{w_g}(f_{w_f}(x_i)),s_i)$ \COMMENT{Calculate the sensitive adversarial loss}
\STATE $\omega_{g} \gets \omega_{g} -\alpha_{g} \frac{\partial  J_s(\omega_{f},\omega_{g})}{\partial {\omega_{g}}}$ \COMMENT{Update the sensitive adversarial model $g_{\omega_{g}}$ by gradient descent}
\ENDFOR
\STATE  $\forall s \in \{0,1\} \quad 
{r}(x_i,s_i) = \frac{\e^{-\mathcal{L}_{S}(g_{w_g}(f_{w_f}(x_i)), s_i)/\tau}}{\frac{1}{n_{s=s_i}}\sum_{(x_j,s_j)\in \Gamma, s_j=s_i}^{n}\e^{-\mathcal{L}_{S}(g_{w_g}(f_{w_f}(x_j)), s_j)/\tau}}$
\COMMENT{Calculate the likelihood ratio per demographic subgroups}
\STATE $J_f(\omega_{f},\omega_{g}) \gets \frac{1}{b}\sum_{i=1}^{b} l_Y(f_{w_f}(x_i),y_i) - \lambda r(x_i,s_i)*l_S(g_{w_g}(f_{w_f}(x_i)),s_i)$ \COMMENT{Calculate the loss function of the predictor model}

\STATE $\omega_{f} \gets \omega_{f} - \alpha_{f}\frac{\partial J_f(\omega_{f},\omega_{g})}{\partial \omega_{f}}$ \COMMENT{Update the predictor model}

\ENDFOR
\ENDFOR
\end{algorithmic}
\end{algorithm}

\end{document}